 \documentclass[final,3p,times]{elsarticle}
\usepackage{amssymb,amsthm,amsmath}
\usepackage{graphicx,psfrag,subfigure}
\usepackage{verbatim}
\usepackage{amsmath,amssymb,amsfonts}
\usepackage{url}
\usepackage{amsfonts}
\usepackage{amssymb}
\usepackage{graphicx}%
\usepackage{amsmath}%
\newcommand{\G}{\mathcal{G}}
\newcommand {\graphset}{\mathbb{G}}
\newcommand{\V}{{\mathcal{V}}}
\newcommand{\E}{{\mathcal{E}}}
\newcommand{\R}{\mathbb{R}}

\newcommand{\MarkovP}{\mathcal{P}}
\newcommand{\mbf}{\mathbf}

\newcommand{\scr}{\mathcal}
\newcommand{\eqdef}{\protect :=}
\DeclareMathOperator{\diag}{diag}
\DeclareMathOperator{\Prob}{Pr}
\DeclareMathOperator{\Exp}{E}
\newtheorem{theorem}{Theorem}
\newtheorem{lemma}{Lemma}
\newtheorem{proposition}{Proposition}
\newtheorem{remark}{Remark}
\graphicspath{{figures/}}

\journal{Computer Communications}

\begin{document}

\begin{frontmatter}



\title{Detecting Separation in Robotic and Sensor Networks}

\author[label1]{Chenda Liao}
\author[label2]{Harshavardhan Chenji}
\author[label1]{Prabir Barooah}
\author[label2]{Radu Stoleru}
\author[label3]{Tam\'{a}s Kalm\'{a}r-Nagy}

\address[label1]{Dept.~of Mechanical and Aerospace Engineering, University of Florida, Gainesville, FL.}
\address[label2]{Dept.~of Computer Science and Engineering, Texas A\&M University, College Station, TX.}
\address[label3]{Dept.~of Aerospace Engineering, Texas A\&M University, College Station, TX.}

%\author{Chenda Liao, Harshavardhan Chenji, Prabir Barooah, Radu Stoleru, Tam\'{a}s
% Kalm\'{a}r-Nagy\thanks{Prabir Barooah is with the Dept of Mechanical and
% Aerospace Engineering, University of Florida, Gainesville, FL. Harshavardhan
% Chenji and Radu Stoleru are with the Dept.~of Computer Science, and Tam\'{a}s
% Kalm\'{a}r-Nagy is with the Dept of Aerospace Engineering, Texas A \& M
% University, College Station, TX.}}

%\address{}

\begin{abstract}
  In this paper we consider the problem of monitoring detecting
  separation of agents from a base station in robotic and sensor
  networks. Such separation can be caused by mobility and/or failure
  of the agents. While separation/cut detection may be performed by passing
  messages between a node and the base in static networks, such a
  solution is impractical for networks with high mobility, since
  routes are constantly changing. We propose a distributed algorithm to
  detect separation from the base station.  The algorithm consists of
  an averaging scheme in which every node updates a scalar state
  by communicating with its current neighbors. We prove that if a node
  is permanently disconnected from the base station, its state
  converges to $0$. If a node is connected to the base station in an
  average sense, even if not connected in any instant, then we show
  that the expected value of its state converges to a positive
  number. Therefore, a node can detect if it has been separated from
  the base station by monitoring its state. The effectiveness of the proposed
  algorithm is demonstrated through simulations, a real system implementation and
  experiments involving both static as well as mobile networks.
\end{abstract}

\begin{keyword}
mobile ad-hoc network  \sep robotic network  \sep sensor network \sep fault detection


\end{keyword}

\end{frontmatter}

\section{Introduction}

Sensor and robotic networks is a quickly developing area extending the
boundaries of traditional robotics and usual sensor networks~\cite{mckee2003nr,MH_etal_JFR:07}. In such a network, static
 as well as mobile nodes (robots) with varying levels of sensing, communication
and actuation capability are used to observe, monitor, and control the
state of physical processes. For example, in a scenario depicted in
Figure~\ref{fig:big_picture}, a team of ground robots may serve as information
aggregators from a large number of static sensors deployed in an area as well
as relays to send the processed data to more maneuverable autonomous aerial
vehicles. We refer to all the devices that take part in sharing information,
whether static or mobile as \emph{nodes} or \emph{agents}. Thus in
Figure~\ref{fig:big_picture} the agents are the chopper, the mobile ground
vehicles and the static sensors.

\begin{figure}[t]
\centering\includegraphics[width=.7\textwidth]{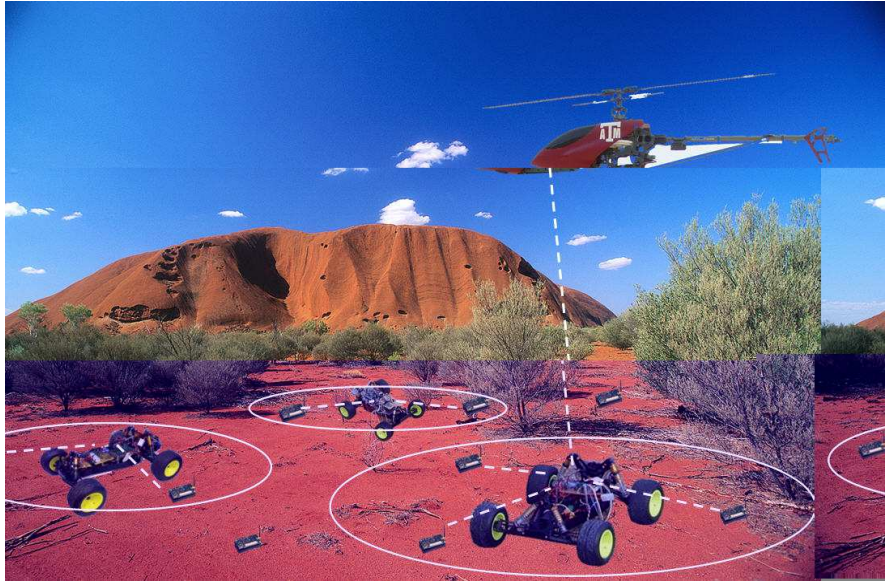}\caption{A
heterogeneous robotic sensor network: ground robots aggregating information
collected from a large number of static sensor nodes and relaying to an aerial
robot. }%
\label{fig:big_picture}%
\end{figure}

The communication topology of a robotic and sensor network is likely
to change over time. Such changes can occur not only due to the
mobility of the robotic nodes, but are also likely with static
agents due to failures. An agent may fail due to various factors such
as mechanical/electrical problems, environmental degradation, battery
depletion, or hostile tampering. These causes are especially common
for networks deployed in harsh and dangerous situations for
applications such as forest fire monitoring, battlefield or emergency
response operations~\cite{kumar2004ras}.

Information exchange through wireless communication is a crucial
ingredient for a sensor and robotic network system to carry out the
tasks it is deployed for. Long range information exchange between
agents is usually achieved by multi-hop wireless communication. In mobile networks, it is quite possible that the agents get
separated into two or more sub-networks with no paths for data routing
among these sub-networks. We say that a \emph{cut} has occurred in such an
event. A node that is disconnected from the base station at some time
may also get reconnected later, e.g., when a mobile node moves in such
a way that it restores connectivity between two disconnected sets of
agents. In a robotic and sensor network, cuts can occur and
disappear due to a combination of node mobility and node failure.

Multi-hop transmission typically requires routing data from a source
node to a sink node, which is usually the base station.  In a network
with highly dynamic topology -- common for sensor and robotic networks
- maintaining updated routing tables, or discovering routing paths on demand, are challenging and energy inefficient
tasks~\cite{Perkins03aodv}. In such situations, sometimes information transfer
from a node to a base station is accomplished by waiting till the node
comes within range of the base station or to another node that is
close to the base station~\cite{Zebranet}. In either scenario, it
is imperative for the agents to know if a cut occurs, so necessary
action can be taken. For example, once a cut is detected, mobile
robots can attempt to heal the network by repositioning themselves or
placing new communication relay nodes in critical regions. There are
other advantages of an ability to detect separation from the base. If
a  node that is separated from the base station initiates data
transmission to the base station, it will only lead to wastage of
precious on-board energy of itself and its nearby nodes that attempt
to route the packets to the base station, since no path to the
destination exists. Therefore, after a cut occurs it is better for
such nodes not to initiate any long-long information transfer. This
requires the nodes to be able to monitor their connectivity to the
base station. In addition, any proposed solution for
separation detection cannot rely on centralized information
processing, since separation will prevent information from reaching the
central node.

A \emph{cut} is defined for static networks as the separation of the
network into two or more disjoint
components~\cite{NS_SS_CT_IPSN:05}. However, for mobile networks we
need to distinguish between a node getting disconnected from the rest
of the nodes temporarily, for a short time interval, from getting
disconnected for a very long time interval, or in the extreme case,
all future time. A node may get disconnected temporarily due to
mobility, or due to the failure of certain nodes, and then reconnected
later. The more dynamic the topology is, the more
likely that disconnections and re-connections will occur
frequently. Therefore, what is needed is an ability of the nodes to detect if it
has been disconnected from the source for a long enough time that will
seriously hamper its ability to send data to the base station if the
need arises. If the disconnection is for a very short period, that is
not a serious cause for concern as the node can simply wait for a
brief period to see if it gets connected again, which may occur due to
the motion of itself or other nodes. We refer to the first as \emph{intermittent disconnection} while the second is
called \emph{permanent disconnection}. The qualifier ``permanent''  is
qualitative, it merely means ``long enough'' to necessitate some
action on the part of the node as discussed earlier.

However, little attention has been paid to the problem of detecting
cuts. Notable exceptions are Kleinberg \textit{et
  al.}~\cite{Kleinberg_IM:03} - who study the problem of detecting
network failures in wired networks -- and Shrivastava \textit{et al.}%
~\cite{NS_SS_CT_IPSN:05} and Barooah~\cite{PB_CDC:08}, who propose
algorithms for detecting cuts in wireless sensor networks. The problem
of detecting cuts in mobile networks has attracted even less
attention. The solutions proposed
in~\cite{Hauspie03partitiondetection,HR_RW_JS_TZ_SIGCOM:04} require
routing packets from the base station to each node periodically. When
a node fails to receive this packet for a certain amount of time, it
suspects that a cut has occurred. This approach, however, requires
routing data between far away nodes, which is challenging in networks
of mobile nodes. While algorithms for coordinating the motion of the
agents to ensure network connectivity has been developed in recent
times (see, for example,~\cite{AH_AC_VK_CT_JFR:07}), such algorithms
cannot guarantee network connectivity under all circumstances. This is
especially true when the robotic agents are operating in harsh and
uncertain environments. Thus, there is a need to develop distributed
algorithms for cut detection in robotic and sensor networks.

In this paper we describe a simple algorithm for cut detection in
robotic and sensor networks. The algorithm is applicable to networks
made up of static or mobile agents, or a combination of the two. This
algorithm -- called Distributed Source Separation Detection (DSSD)
algorithm -- is designed to allow every node to monitor if it is
connected to a specially designated node, the so-called \emph{source
  node}, or if it has been disconnected from the source node. The
source node is usually a base station. The reason for this terminology
comes from an electrical analogy that the algorithm is based on.  The
idea is quite simple: imagine the wireless communication network as an
electrical circuit where current is injected at the source node. When
a cut separates certain nodes from the source node, the potential at
these nodes becomes zero. If a node is connected to the source node
through multi-hop paths, either always or in a time-average sense, the
potential is positive in a time-average sense. In the DSSD algorithm,
every nodes updates a scalar (called its \emph{state}) which is an
estimate of its virtual potential in the fictitious electrical
network. The state update is performed by a distributed algorithm that
only requires a node to communicate to its neighbors. By monitoring
their computed estimates of the potential, nodes can detect if they
are connected to the source node or not. In addition, the iterative
scheme used in computing the node potentials is extremely fast, which
makes the algorithm scalable to large networks. The proposed DSSD
algorithm is fully distributed and asynchronous. It requires
communication only between neighbors, multi-hop routing between node
pairs that are not neighbors is not needed.

Performance of the algorithm in networks of static nodes was examined
through simulations previously in~\cite{PB_CDC:08}. Here we extend the
algorithm to both static and mobile networks. In the mobile case, by
modeling the evolution of the network over time as a Markov chain, we
show that the node states converge to positive numbers in an expected
sense as long a path exists in a time-average sense between the node
and the base station. The performance of the algorithm has been tested
in simulations as well as in an experimental platform with mobile
robots and human agents. These tests demonstrate the effectiveness of
the algorithm in detecting separation (and reconnection) of nodes in
both static and mobile networks. Since there is no existing prior work on the problem of
detecting separation in mobile networks that can operate without
multi-hop routing, we donot present comparison of the proposed
algorithm with existing algorithms for cut detection.

% The algorithm proposed by Shrivastava \textit{et al.}
% in~\cite{NS_SS_CT_IPSN:05} can detect $\epsilon$-linear cuts, which is a
% separation of the network across a straight line so that at least $\epsilon n$
% of the nodes ($n$ is the total number of nodes in the network) are separated
% from a base station. The algorithm is applicable to planar geometries and
% require communication between a base station and a set of sentinel nodes. In
% contrast, the DSSD algorithm proposed here can detect cuts that separate the
% network into multiple components of arbitrary shapes and is not limited to 2D
% geometries. Furthermore, the DSSD algorithm is applicable to both static and
% mobile networks.

The rest of the paper is organized as follows. In Section~\ref{SECTION_DSSD}
we introduce the DSSD algorithm. The  rationale behind the algorithm
and its theoretical properties are described in Section~\ref{SECTION_theory}.
Sections~\ref{SECTION_Simulations} and~\ref{sec:system_impl} describe results
from computer simulations and experimental evaluations.
The paper concludes with a summary in Section~\ref{SECTION_Conclusions}.

\section{The Distributed Source Separation Detection (DSSD) Algorithm}\label{SECTION_DSSD}
We introduce some terminology about graphs that will be needed to
describe the algorithm and the results precisely (see for example the
excellent treatise by Diestel \cite{diestel2000gt}). Given a set
${\mathcal{V}}=\{v_{1},\ldots,v_{m}\}$ of $m$ elements referred to as
\emph{vertices}, and a set ${\mathcal{E}}=\{(v_{i},v_{j})\,|\,v_{i}%
,v_{j}\in{\mathcal{V}}\}$ of \emph{edges}, a \emph{graph}
$\mathcal{G}$ is defined as the pair
$({\mathcal{V}},{\mathcal{E}})$. A sensor network is modeled as a
graph $\mathcal{G}=({\mathcal{V}},{\mathcal{E}})$ whose vertex set
${\mathcal{V}}$ corresponds to the wireless sensor nodes and whose
edges ${\mathcal{E}}$ describe direct communication between nodes.
The size of a graph ${\mathcal{G}}$ is the number of its vertices
$|{\mathcal{V}}|$. The graphs formed by the nodes of the sensor and
robotic network are assumed to be \emph{undirected}, i.e. the
communication between two nodes is assumed to be symmetric. In the
language of graph theory, the edges of an undirected graph are
unordered pairs with the symmetry relation
$(v_{i},v_{j})=(v_{j},v_{i})$. The \emph{neighbors} of vertex $v_{i}$
is the set $\mathcal{N}_{i}$ of vertices connected to $v_{i}$ through
an edge, i.e.
$\mathcal{N}_{i}=\{v_{j}|(v_{i},v_{j})\in{\mathcal{E}}\}$. The number
of neighbors of a vertex $\left\vert \mathcal{N}_{i}\right\vert $ is
called its \emph{degree}. A graph is called \emph{connected} if for
every pair of its vertices there exists a sequence of edges connecting
them.

In a mobile sensor and robotic networks the neighbor relationship can
change with time, so the graph in our study are in general time
varying: $\G(k) = (\V,\E(k)$, where $k=0,1,\dots$ is a discrete time
index. Note that we assume the set $\V$ of nodes does not change over
time; though certain nodes may fail permanently at some time and
thereafter not take part in the operation of the network.

In the proposed scheme, every node $v_i$ maintains a scalar variable
$x_i(k)$ in its local memory, which is called its \emph{state}. The base
station is designated as the \emph{source node}, though in principle
any node can be the source node. The reason for this terminology will
be explained soon. For ease of description (and without loss of
generality), the index of the source node is taken as $1$. The
DSSD algorithm is an iterative process consisting of two phases at
every discrete step: (i) \emph{State Update} and (ii) \emph{Cut
  Detection from State}.

DSSD PHASE I (\textbf{State update law)}: The scalar state
$x_{i}(k)$ assigned to node $v_{i}$ is iteratively updated (starting with
$x_{i}\left(  0\right)  =0$) according to the following update
rule. Recall that the index $i=1$ corresponds to the source node.
\begin{equation}
x_{i}(k+1)=\left\{
\begin{array}
[c]{c}%
\frac{1}{d_{1}(k)+1}\left(
{\displaystyle\sum\limits_{v\in\mathcal{N}_{1}(k)}}
x_{i}(k)+s\right) \\
\frac{1}{d_{i}(k)+1}%
{\displaystyle\sum\limits_{v\in\mathcal{N}_{i}(k)}}
x_{i}(k)
\end{array}
\right.
\begin{array}
[c]{c}%
i=1\\
i>1
\end{array}
. \label{eq:stateupdatelaw}%
\end{equation}
where $\mathcal{N}_{i}(k)$ is the set of neighbors
of $v_{i}$ in graph $\mathcal{G}(k)$ and
$d_{i}(k):=|\mathcal{N}_{i}(k)|$ is the degree of $v_{i}$ (number of neighbors) at time $k$, and the $s>0$ is an arbitrary fixed positive
number that is called the \emph{source strength}. The source strength
is a design parameter and has to be provided to the source node a-priori.

DSSD PHASE II (\textbf{Cut detection from state}): Every node $v_i$
maintains determines its connectivity to the source node by
comparing its state $x_i(k)$ to the \emph{cut detection threshold}
$\epsilon$ (a small positive number) as follows:
\begin{align}
  \label{eq:1}
  \mathrm{cut\_belief}_i(k) =
  \begin{cases}
    0 & x_i(k) > \epsilon \\
    1 & x_i(k) \leq \epsilon
  \end{cases}
\end{align}
where $\mathrm{cut\_belief}_i=1$ means the node believes it is
disconnected from the source and $0$ means it  believes it is
connected to the source.

The rationale for the algorithm comes from the interpretation of the
states in terms of the potentials in an electrical circuit. If the
network does not change with time, then the state of a node that is
connected to the source converges to positive number that is equal to
its electrical potential in a fictitious electrical network. If a node
is disconnected from the source then its state converges to $0$. When
the network is time-varying, then too the states can be shown to
converge in a mean-square sense to either positive numbers or $0$
depending on whether the node is connected or not (in some appropriate
stochastic sense) to the source. We discuss the details of the
electrical analogy and the theoretical performance guarantees that can
be provided for the proposed algorithm in the next section.

We note that the cut detection threshold $\epsilon$ is a design parameter, and it
has to be provided to all the nodes a-priori.  The value of
$\epsilon$ chosen depends on the source strength $s$. Smaller the
value of $s$, the smaller the value of  $\epsilon$ that has to be
chosen to avoid false separation detection.   We also note that the
algorithm as described above assumes that all updates are done synchronously, or, in
other words, every node shares the same iteration counter $k$. In
practice, the algorithm is executed asynchronously without requiring
any clock-synchronization or keeping a common time counter.
To achieve this, every node keeps a buffer of the last received states of its
neighbors. If a node does not receive messages from a neighbor during a
time-out period, it updates its state using the last successfully received
state from that neighbor. When a node does not receive broadcasts from one of
its neighbors for sufficiently long time, it removes that neighbor from its
neighbor set and carries on executing the algorithm with the remaining
neighbors. 

\section{Algorithm Explanation and Theoretical Results}\label{SECTION_theory}
When the graph does not change with time, i.e., $\G(k) = \G$ for some
$\G$, the state update law is an iterative method for computing the
node potentials in a fictitious electrical network $\G^e =
(\V^e,\E^e)$ that is constructed from the graph $\G$ as
follows. First, define $\V^e \eqdef V\cup\{g\}$ where $g$ is a
fictitious \emph{grounded node} $g$ and next, introduce $n$ additional
edges to connect each of the $n$ node in $\V$ to the node $g$ with a
single edge. So the edges in $\E^e$ consist of the edges in $\E$ and
the newly introduced edges. Now an electrical network $(\G^e, 1)$ is
imagined by assigning to every edge of $\G^e$ a unit
resistance. Figure~\ref{fig:example-network} shows a physical network
and the corresponding fictitious electrical network. It can be shown
that in a time invariant network, the node states resulting from the
state update law always converge to constants~\cite{PB_CDC:08}. In
fact, the limiting value of a node state in a graph $\G$ is the
potential of the node in the fictitious electrical network $\G^e$ in
which $s$ Ampere current is injected at the source node and is
extracted at a fictitious grounded node; which is always maintained at
potential $0$ by virtue of being grounded (see Theorem 1
of~\cite{PB_CDC:08}).

\begin{figure}[t]
\begin{center}
\psfrag{s}{$s$ A} \includegraphics[scale = 0.4]{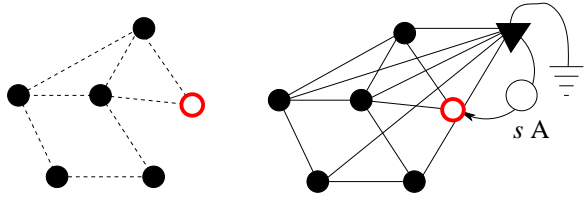}
%wireless sensor network, fictitious electrical network
\end{center}
\caption{A graph describing a sensor network (left), and the associated
electrical network (right). In the electrical network, one node is chosen as
the source that injects $s$ Ampere current into the network, and additional
nodes are introduced (fictitiously) that are grounded, through which the
current flows out of the network. The thick line segments in the electrical
network are resistors of $1 \Omega$ resistance. }%
\label{fig:example-network}%
\end{figure}

The evolution of the node states for a \emph{static network of nodes}
was analyzed in~\cite{PB_CDC:08}, where it was shown that for a time
invariant graph that is connected (and therefore every node is
connected to the source), the state of every node converges to a
positive number (see Theorem 1 of~\cite{PB_CDC:08}). For nodes that
are disconnected from the source, it was shown that their states
converge to $0$, and in fact this result holds even if the
component(s) of nodes that are disconnected from the source are
changing with time due to mobility etc. We state the precise result
below:
\begin{theorem}
{~\cite{PB_CDC:08}}\label{thm:cut-state-evol}
Let the nodes of a sensor network
with an initially connected undirected graph $\mathcal{G}=({\mathcal{V}%
},{\mathcal{E}})$ execute the DSSD algorithm starting at time $k=0$ with
initial condition $x_{i}(0)=0,\;i=1,\ldots,|\mathcal{V}|$. If a node
$v_i$ gets disconnected from the source at a time $\tau>0$ and stays
disconnected for all future time, then its state $x_i(k)$ converges to
$0$ as $k\to \infty$.
\end{theorem}

This result is useful in detecting disconnection from the source; if
the state of node converges to $0$ (which can be determined by
checking if it becomes smaller than a threshold), then the node
detects that it is disconnected from the source. This partially explains the
logic behind the Phase II of the algorithm.

To detect connectivity to the source, we have to answer the question:
how do the states of nodes that are \emph{intermittently connected} to
the source due to their own -- or other nodes' -- mobility evolve? In
a mobile network the graph at any time instant is essentially random
since it depends on which nodes are within range of which other nodes
at that instant, which is difficult to predict in all but the most
trivial situations. We therefore model the evolution of the graph
$G(k)$ as a stochastic process. In that case the state of every node
also is a random variable, whose value depends on the evolution of the
graphs. Assuming the evolution of the graph can be described by a Markov
chain, we then show that the node states converge in the mean square
sense. Meaning, the mean and variance of each node's state converge to
specific values. We also provide formulas for these values.

% ----
% The reason the state of a node is useful in detecting connectivity to
% the source can be understood from the electrical analogy described above. If a node
% $v$ is disconnected from the source in $\G$, then in the fictitious
% electrical network $\G^e$ there is no path from the source node to $v$
% except through the grounded node. As a result no current passes
% through the node $v$; its potential is is the same as that of the
% grounded node, which is $0$. Since the node states converge to the
% potentials in the fictitious
% electrical network $\G^e$, the state of node $v$ converges to $0$. If
% however, $v$ is connected to the source then its potential converges
% to a positive number (its potential in the electrical network), so the
% node can determine that it is indeed connected to the source node from
% its own state.
% ---------

We consider the case when the sequence of graphs
$\{\G(k)\}_{k=0}^{\infty}$ can be modeled as the realization of a
Markov chain, whose state space $\graphset \eqdef \{\G_1,\dots,\G_N\}$
is the set of graphs that can be formed by the mobile nodes due to
their motion. The network at time $k$ can be any one of the elements
of the set $\graphset$, i.e., $\G(k) \in \graphset$. The Markovian
property means that if $\G \in \graphset$, then $\Prob(\G(k+1) = \G |
\G(k)) = \Prob(\G(k+1) = \G | \G(k),\G(k-1),\dots,\G(0))$, where
$\Prob(\cdot)$ denotes probability. A simple example in which the time
variation of the graphs satisfies the Markovian property is that of a
network of static nodes with unreliable communication links such that
each link can fail temporarily, and the failure of each edge at every
time instant $k$ is independent of the failures of other links and the
probability of its failure is time-invariant. Another example is a
network of mobile agents whose motion obeys first order dynamics with
range-determined communication. Specifically, suppose the position of
node $v_i$ at time $k$, denoted by $p_i(k)$, is restricted to lie on
the unit sphere $\mathbf{S}^2 = \{ x \in \R^3| \|x\|=1\}$, and suppose
the position evolution obeys: $p_i(k+1) = f(p_i(k) + \Delta_i(k))$,
where $\Delta_i(k)$ is a stationary zero-mean white noise sequence for
every $i$, and $\Exp[\Delta_i(k) \Delta_j(k)^T] = 0$ unless
$v_i=v_j$. The function $f(\cdot): \R^3 \to \mathbf{S}^2$ is a
projection function onto the unit-sphere. In addition, suppose
$(v_i,v_j) \in \E(k)$ if and only if the geodesic distance between
them is less than or equal to some predetermined range. In this case,
prediction of $\G(k+1)$ given $\G(k)$ cannot be improved by the
knowledge of the graphs observed prior to $k$: $\G(k-1),\dots,\G(0)$,
and hence the change in the graph sequence satisfies the Markovian
property. If no restriction is placed on the motion of the nodes or edge
formation, the number of graphs in the set $\graphset$ is the total
number of distinct graphs possible with $n$ nodes. In that case, $N =
\displaystyle{2^{\frac{1}{2}n(n-1)}}$, where $N \eqdef
|\graphset|$. If certain nodes are restricted to move only within
certain geographic areas, $N$ is less than this maximum number.

We assume that the Markov chain that governs the evolution of the
graphs $\{G(k)\}_{k=0}^{\infty}$ is homogeneous, and denote the
transition probability matrix of the chain by $\MarkovP$. The
following result states under what conditions the node states converge
in the mean square sense and when the limiting mean values are
positive. The proof of the result is provided in Section~\ref{sec:proofs}. In the statement of theorem, the \emph{state vector} is the
vector of all the node states: $\mbf{x}(k) \eqdef
[x_1(k),\dots,x_{n}(k) ]^T$, and the \emph{union graph} $\hat{\G}
\eqdef \cup_{i=1}^{N}\G_i$ is the graph obtained by taking the union
of all graphs in the set $\graphset$, i.e., $\hat{\G} =
(\V,\cup_{i=1}^{N}\E_i)$.

\begin{theorem}\label{thm:mobile-convergence}
When the temporal evolution of the graph $\G(k)$ is governed by a
Markov chain that is ergodic, the state vector $\mbf{x}(k)$ converges in the mean
square sense. More precisely, for every initial condition $\mbf{x}(0)$, there exists
vector $\mu \in \R^n$ and a symmetric positive semi-definite matrix $\mbf{Q}$
so that $\Exp[\mbf{x}(k)] \to \mu$ and $\Exp[\mbf{x}(k) \mbf{x}(k)^T]
\to \mbf{Q}$. Moreover, the vector $\mu$ is entry-wise
non-negative. If $P$ is entry-wise positive, we have $\mu(i) > 0$
if and only if there is a path from node $v_i$ to the source node
$v_1$ (with index 1) in the union graph $\hat{\G}$.
\end{theorem}

Theorems~\ref{thm:mobile-convergence}
and~\ref{thm:cut-state-evol} together explain the rationale behind
Phase II of the DSSD algorithm. The state of a node converges to $0$ if and
only if it is permanently disconnected from the source. If it is
connected to the source in the union graph, meaning it is connected in
a time-average sense (even if it is not connected in every time
instant), then the expected value of its state converges to a positive
number. As a result, a node can detect if it is connected to the
source or not simply by checking if its state has converged to $0$,
which is done by comparing the state to the threshold $\epsilon$.

\begin{figure}[t]
\begin{center}
\hspace{.5in} \subfigure[$\G_1$]{\includegraphics[scale= 0.5, clip
  =true, trim = 0in 0in -1in 0in]{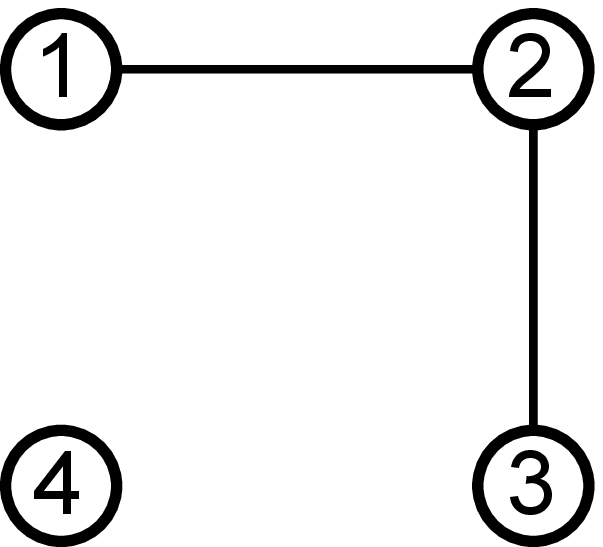}}
\subfigure[$\G_2$]{\includegraphics[scale= 0.5, clip
  =true, trim = -1in 0in 0in 0in]{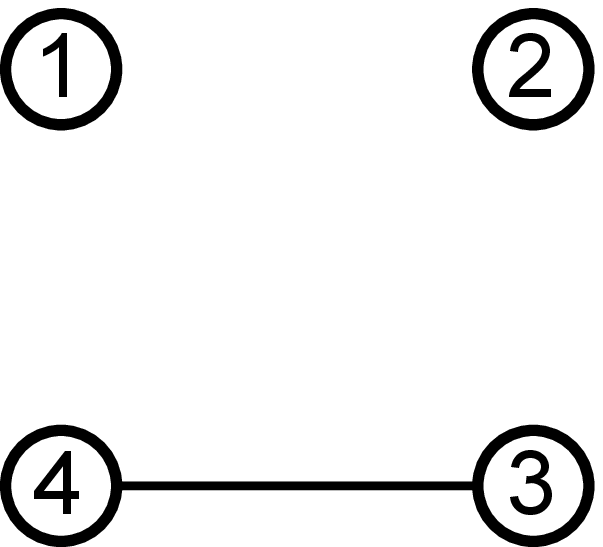}}
\end{center}
\caption{An example of two graphs $\G_1$ and $\G_2$ that appear at
  different times due to the motion of $4$ mobile nodes. If these two
  graphs form the set of all possible graphs that can ever appear,
  then $\graphset = \{\G_1,\G_2\}$. Note that neither $\G_1$ nor
  $\G_2$ is connected, but the union graph $\cup_{i}\G_i$ is. }
\label{fig:Markov-graphs}
\end{figure}

A closed-form expression for the limiting mean of the nodes states,
i.e., the vector $\mu$, and its correlation matrix $\mbf{Q}$, is also
provided in Lemma~\ref{lem:ss-var} in Section~\ref{sec:proofs}. We refrain from
stating them here as the expressions require
significant amount of terminology to be introduced.  The ergodic
property of the Markov chain assumed in
Theorem~\ref{thm:mobile-convergence} ensures that the chain has a
steady state distribution which is an entry-wise positive
vector. Intuitively, ergodicity of the chain means
that every graph in the set $\graphset$ appears infinitely often as
time progresses. In other words, the network $\G(k)$ does not get
stuck in one of the graphs in the set $\graphset$. As a result, if a
node is connected to the source in the union graph, even if it is not
connected to the source in any one of the graphs that can ever appear, there is
still a path for information flow between the source node and this node over
time. Figure~\ref{fig:Markov-graphs} shows a simple example in which
the network $\G(k)$ can be only one of the two graphs shown in the
figure. Node $4$ is disconnected from the source at every $k$, but it
is connected in a time-average sense since information from the source
node $1$ can flow to $2$ in one time instant when the graph $\G_1$
occurs and then from $2$ to $4$ in another time instant when $\G_2$
occurs. In a such a situation the theorem guarantees that the expected
value of the state of the node $4$ will converge to positive number.
Thus, node $4$ can detect that it is connected to the source in a
time-average sense. On the other hand, if a node is not connected to
the source in the union graph there is no path for information flow
between itself and the source over all time. This is equivalent to the
node being permanently disconnected from the source, so the result
that the mean of the node's state converges to $0$ is consistent with
the result of Theorem~\ref{thm:cut-state-evol}. The condition that
$\MarkovP$ is entry-wise positive means that there is a non-zero
probability that the network at time $k$ can transition into any one
of the possible graphs at time $k+1$, even though most of these
probabilities maybe small. We believe this sufficient condition is
merely an artifact of our proof technique, and in fact, this is not a
necessary condition for convergence to occur.

\begin{remark}
  \label{rem:speed} An advantage of the DSSD algorithm is that its
  convergence rate for a time-invariant network is independent of the
  number of agents in the network (Lemma 1 in~\cite{PB_CDC:08}). This
  makes the algorithm scalable to large networks. The
  convergence rate of mean and variance of the nodes states to their
  limiting values in the mobile case, however, requires further research.
\end{remark}

\subsection{Proof of Theorem~\ref{thm:mobile-convergence}}\label{sec:proofs}
We start with some notation. let $D(k)$ be the \emph{degree matrix}
of $\G(k)$, i.e., $D(k) \eqdef \diag(d_1(k),\dots,d_{n}(k))$. If node
$i$ fails at $k_0$, we assign $d_i(k) = 0$ and $\scr{N}_i(k)=\phi$
(the empty set) for $k \geq k_0$.  Let $A(k)$ be the \emph{adjacency
  matrix} of the graph $\G(k)$, i.e., $A_{i,j}(k) = 1$ if $(i,j)\in \E(k)$, and $0$ otherwise.
With these matrices,~\eqref{eq:stateupdatelaw} can be compactly written as:
\begin{align}\label{eq:state-update-law-vector}
  \mbf{x}(k+1) & = (D(k) + I)^{-1}\left(A(k) \mbf{x}(k) + s\;e_1\right)
\end{align}
where $e_1 =[1,0,\dots,0]^T$. Recall that the source node has
been indexed as node $v_1$. The above can be written as
\begin{align}\label{eq:update-vector-2}
\mbf{x}(k+1) & = J(k) \mbf{x}(k) + B(k)w(k)
\end{align}
where the matrices $J,B$ and the vector $w$ are defined as
\begin{align}\label{eq:J-B-w-def}
J(k) & \eqdef (D(k) + I)^{-1}A(k), & B(k) & = (D(k) + I)^{-1}, &
  w(k) & \eqdef s e_1.
\end{align}
Under the assumption that the temporal evolution of the graphs $\G(k)$
is governed by a Markov chain, we can
write~\eqref{eq:state-update-law-vector} in the standard notation of
\emph{jump-linear system}~\cite{CostaFragosoMarques:04}
\begin{align}\label{eq:update-law-JLS}
\mbf{x}(k+1) = J_{\G(k)} \mbf{x}(k) + B_{\G(k)} w(k)
\end{align}
 where $J_{\G(k)} = J(k)$, $B_{\G(k)} = B(k)$, and $w(k) = w=se_1$. This
 notation is used to emphasize that the state and input matrices $J$
 and $B$ of the linear system~\eqref{eq:update-law-JLS} change
 randomly, and the transition is governed by a Markov chain.

Let $\mu(k)\eqdef\Exp[\mbf x(k)]$, $\mathbf{Q}(k)\eqdef \Exp[\mbf
 x(k){\mbf x(k)}^T]$ be the mean and correlation of the state vector $\mbf{x}(k)$, respectively. We need the following definitions
 and terminology from~\cite{CostaFragosoMarques:04} to provide
 expressions for these quantities as well as to state conditions for
 the convergence of the mean and correlation. Let
 $\mathbb{R}^{m \times n}$ be the space of $m\times n$ real
 matrices. Let $\mathbb{H}^{m \times n}$ be the set of all N-sequences
 of real $m\times n$ matrices $Y_i\in \mathbb{R}^{m\times n}$. That
 is, if $Y\in \mathbb{H}^{m \times n}$ then $Y = (Y_1,Y_2,\dots,Y_N)$,
 where each $Y_i$ is an $m \times n$ matrix. The operators $\varphi$
 and $\hat{\varphi}$ are defined as follows: let ${(y_i)}_j\in \R^{m}$ be
 the $j$-th column of $Y_i \in \R^{m \times n}$, then
\begin{align}\label{eq:tall_matrix}
  \varphi(Y_i)& \eqdef\left(\begin{array}{c} {(y_i)}_1\\\vdots\\{(y_i)}_n\end{array}\right)
&\text{ and } & &\hat{\varphi}(Y) & \eqdef\left(\begin{array}{c}\varphi(Y_1)\\\vdots\\\varphi(Y_N)
\end{array}\right)
\end{align}
Hence, $\varphi(Y_i) \in \mathbb{R}^{mn}$ and
$\hat{\varphi}(Y)\in\mathbb{R}^{Nmn}$. Similarly, define an inverse
function $\hat{\varphi}^{-1}: \R^{Nmn} \to \mathbb{H}^{m \times n}$ that produces an element of
$\mathbb{H}^{m \times n}$ given a vector in $\R^{Nmn}$. For $X_i\in\mathbb{R}^{n
  \times n}$ for $i=1,\dots,N$, define
\begin{align}
diag[X_i]\eqdef\left(\begin{array}{ccc}
    X_1&\ldots&0\\\vdots&\ddots&\vdots\\0&\ldots &X_N
    \end{array}\right)\in \mathbb{R}^{Nn \times Nn}.
\end{align}
For a set of square matrices $C_{ij} \in \R^{m \times m}$,
$i,j = 1,\dots,N$, we also use the notation $C=[C_{ij}]$ to
denote the following matrix:
\begin{align*}
  C =[C_{ij}] \eqdef
\begin{bmatrix}
C_{11} & C_{12} & \dots & C_{1N} \\
C_{21} & C_{22} & \dots & C_{2N} \\
\dots & \dots & \dots & \dots \\
C_{N1} & C_{N2} & \dots & C_{NN}
\end{bmatrix} \in \R^{Nm \times Nm}
\end{align*}

\medskip

In context of the jump linear system~\eqref{eq:update-law-JLS}, define the matrices
\begin{align}\label{eq:definition_A_2}
 \mathcal{C}&\eqdef(\MarkovP^T\otimes I_n)diag[J_i]\in \mathbb{R}^{Nn}\nonumber\\
 \mathcal{D}&\eqdef(\MarkovP^T\otimes I_{n^2})diag[J_i\otimes J_i]\in \mathbb{R}^{Nn^2},
\end{align}
where $I_n$ is the $n\times n$ identity matrix, $\MarkovP$ is
the transition probability matrix of the Markov chain and $\otimes$
denotes Kronecker product. Note that the matrices $\mathcal{C}$ and
$\mathcal{D}$ can be expressed, using the notation introduced above, as
\begin{align}\label{eq:C-D-bracket}
  \mathcal{C} & = [p_{ji}J_j] & \mathcal{D} & = [p_{ji}F_j] \text {
    where } F_i \eqdef J_i \otimes J_i,
\end{align}
where $p_{ij}$ is $(i,j)$-th entry of $\MarkovP$ and $\pi \in \R^{1
  \times N}$ is the stationary distribution of the Markov chain, which
exists due to the ergodicity directly followed by assumption of
positive $\MarkovP$.

%For the proof of Theorem~\ref{thm:mobile-convergence}, we will need
%the following terminology.
For a matrix $X$, we write $X\succeq 0$ to
mean that $X$ is entry-wise non-negative and write $X\succ 0$ to mean
$X\succeq 0$ and $X \neq 0$. If every entry of $X$ is positive, we
write $X \succ\!\succ 0$. For a matrix $X$, $X \geq (>) 0$ means it is
positive semi-definite (definite). For two matrices $X$ and $Y$ of
compatible dimension, we write $X \succeq Y$ if $X_{ij} \geq Y_{ij}$
for all $i$, $j$, and write $X \succ Y$ if $X \succeq Y$ and $X \neq
Y$. For a vector $x$, we write $x \succeq 0$ to mean $x$ is entry-wise
non-negative, $x \succ 0$ to mean $x$ is entry-wise non-negative and
at least one entry is positive, and $x \succ\!\succ 0$ to mean every
entry of $x$ is positive. The fact that both $J$ and $B$ are entry-wise non-negative will be
useful later.

We will also need the following technical results to prove Theorem~\ref{thm:mobile-convergence}.
\begin{proposition}[\cite{MQC_XZL:04}, Theorem 3.2]\label{prop:JLS-MSS-estimate-radius}
Let $C=[C_{ij}] \in \R^{mn \times mn}$ be a block $m \times m$ matrix, where
$C_{ij}$ are non-negative $n\times n$ matrices for all
$i,j=1,\dots,m$ and let $\tilde{C}=[\|C_{ij}\|]
\in \R^{m \times m}$, where $\| \cdot \|$ is either the induced
1-norm ($\|\cdot\|_1$) or the induced $\infty$-norm ( $\|\cdot\|_\infty$). Then $\rho(C) \leq \rho(\tilde{C})$.
\end{proposition}

% \begin{proposition}\label{prop:rho-Ji}
% Consider the matrix $J_i \eqdef (D_i + I)^{-1}A_i$, where $D_i$ and
% $A_i$ are the degree and adjacency matrices for the $i$-th graph,
% $i=1,\dots,|\graphset|$ (i.e., it is the matrix $J(k)$
% in~\eqref{eq:J-B-w-def} if the graph $\G(k)$ at the $k$-th time instant
% happens to the $i$-th graph of the set $\graphset$). Then
% $\rho(J_i)<1$, and in particular, $\rho(J_i) < 1-
% \frac{1}{1+d_{\max}^{(i)}}$, where $d_{\max}^{(i)}$ is the maximum node
% degree in the graph $\G_i$.
% \end{proposition}
% \begin{proof}
% Follows immediately from Lemma 1 of~\cite{PB_CDC:08}.
% \end{proof}

\begin{proposition}\label{prop:Fi-infty-norm}
Let $\G$ be an undirected graph with $n$ nodes, and let
$J=(D+I)^{-1}A$, where $D$ and $A$ are the degree and adjacency
matrices, respectively. $\|J\|_\infty <1$ and $\|J \otimes J \|_\infty <
1$.
\end{proposition}
\begin{proof}
Based on examining the structure of $J$, we see that
$\displaystyle{J_{ij}=\frac{1}{d_i+1}1_{\{\scr{N}_i\}}(j)}$, where
$1_{\{A\}}(x)$ is $1$ if $x \in A$ and $0$ otherwise. Obviously, $J$ is non-negative matrix and so is $J \otimes J $. Meanwhile, each row
sum of $J$ is at most $n-1/n$, where $n$ is the number of
nodes. Therefore $\|J \|_\infty \leq \frac{n-1}{n}<1$, which leads to
$\|J \otimes J \|_\infty=\|J\|_\infty \|J \|_\infty   <1$, where the
equality is a result of properties of the Kronecker
product~\cite[Section 2.5]{AL_WS_04}.
\end{proof}

%%%%%%%%%%%%% important result on spectral radius of C and D %%%%%%%%
\begin{proposition}\label{prop:rho_of_CD}
If the temporal evolution of the graph $\G(k)$ is governed by a
Markov chain that is ergodic, then we have $\rho(\mathcal{C})<1$ and $\rho(\mathcal{D})<1$, where $\mathcal{C},\mathcal{D}$ are defined in~\eqref{eq:definition_A_2}.
\end{proposition}
\begin{proof}
Since $\rho(\mathcal{C})
=\rho([p_{ji}J_j])$ (see~\eqref{eq:C-D-bracket}), we obtain by applying
Proposition~\ref{prop:JLS-MSS-estimate-radius} that
 \begin{align*}
    \rho(\mathcal{C})  &\leq \rho(\left[\| p_{ji}J_j \|_\infty
    \right]) = \rho(\left[ p_{ji}\|J_j\|_\infty \right]),
\end{align*}
where the equality follows from $p_{ij}$'s being probabilities and therefore non-negative. Since $\|J \|<1$ (Proposition~\ref{prop:Fi-infty-norm}), it follows that $P^T
\succ [p_{ji}\|J_j\|_\infty]$. Since both $P^T$ and $[p_{ji}\|J_j\|_\infty]$ are
non-negative, and $P^T$ is irreducible (which follows from the ergodic
assumption of the Markov chain), it follows from Corollary 1.5
of~\cite[pg. 27]{BermanPlemmons:79} that $\rho([p_{ji}\|J_j\|_\infty])
< \rho(\MarkovP^T) = \rho(\MarkovP) = 1$, the last equality being a
property of a transition probability matrix. This proves that
$\rho(\mathcal{C})<1$.

To show that $\rho(\mathcal{D})<1$, since $\rho(\mathcal{D})
=\rho([p_{ji}F_j])$ (see~\eqref{eq:C-D-bracket}), we obtain by applying
Proposition~\ref{prop:JLS-MSS-estimate-radius} that
 \begin{align*}
    \rho(\mathcal{D})  &\leq \rho(\left[\| p_{ji}F_j \|_\infty
    \right]) = \rho(\left[ p_{ji}\|F_j\|_\infty \right]),
\end{align*}
where the equality follows from $p_{ij}$'s being probabilities and therefore non-negative. Since the scalars $\|F_j\|_\infty$
satisfy $\|F_j\|_\infty <1$ for each $j$ (see Proposition~\ref{prop:Fi-infty-norm}), it follows that $P^T
\succ [p_{ji}\|F_j\|_\infty]$. Since both $P^T$ and $[p_{ji}\|F_j\|_\infty]$ are
non-negative, and $P^T$ is irreducible (which follows from the ergodic
assumption of the Markov chain), it follows from Corollary 1.5
of~\cite[pg. 27]{BermanPlemmons:79} that $\rho([p_{ji}\|F_j\|_\infty])
< \rho(\MarkovP^T) = \rho(\MarkovP) = 1$, the last equality being a
property of a transition probability matrix. This proves that
$\rho(\mathcal{D})<1$.
\end{proof}

The proof of Theorem~\ref{thm:mobile-convergence} will require the
following result.
\begin{lemma}\label{lem:ss-var}
Consider the jump linear system~\eqref{eq:update-law-JLS} with the
an underlying Markov chain that is ergodic. If $\rho(\mathcal{D})<1$,
where $\mathcal{D}$ is defined in~\eqref{eq:definition_A_2},
then the state $\mbf{x}(k)$ of the system~\eqref{eq:update-law-JLS}
converges in the mean square sense, i.e., $\mu(k)
  \to \mu$ and $\mbf{Q}(k) \to Q$, where $\mu$ and $Q$ are given by
\begin{align}\label{eq:mu-Q}
  \mu & \eqdef\displaystyle\sum_{i=1}^{N}q^{(i)} \quad
\mathbf{Q} \eqdef\displaystyle\sum_{i=1}^{N}Q_i.
\end{align}
where
\begin{align*}
[{q^{(1)}}^T,\dots,{q^{(N)}}^T]^T = q & \eqdef (I-\mathcal{C})^{-1}\psi \quad (q \in \mathbb{R}^{Nn})\\
(Q_1,\ldots,Q_N) = Q & \eqdef \hat{\varphi}^{-1}\left(
  (I-\mathcal{D})^{-1}\hat{\varphi}(R(q))\right), \quad (Q \in \mathbb{H}^{n \times n})
\end{align*}
where
\begin{align*}
\psi &\eqdef [\psi_1^T, \dots, \psi_N^T]^T
\in \mathbb{R}^{Nn} && \text{ and } &\psi & \eqdef
\displaystyle\sum_{i=1}^{N}p_{ij}B_iw\pi_i \in \R^n \\
R(q) &\eqdef(R_1(q),\ldots,R_N(q))\in \mathbb{H}^{n \times n} && \text{
  and } & R_j(q) & \eqdef \displaystyle\sum_{i=1}^{N}p_{ij}(B_iww^T B_i^T\pi_i+J_iq^{(i)}
w^T B_i^T+B_iw {q^{(i)}}^TJ_i^T))\in \mathbb{R}^{n
  \times n},
\end{align*}
and $\mathcal{C}$ is defined in~\eqref{eq:definition_A_2}, $\pi_i$ is the $i$-th entry of the
steady state distribution of the Markov chain, and $J_i,B_i$ are the
system matrices in~\eqref{eq:update-law-JLS}. Moreover, $\mu \succeq 0$.
\end{lemma}
\begin{proof}
The first statement about mean square convergence follows from
standard results in jump linear systems, as do the expressions for the
mean and correlation; see~\cite[Proposition
3.37]{CostaFragosoMarques:04}. Note that the existence of the steady
state distribution $\pi$ follows from the ergodicity of the
Markov chain.

To show that $\mu$ is entry-wise non-negative, note that
since $\rho(\mathcal{C})<1$ (Proposition \ref{prop:rho_of_CD}), we
have $\mathcal{M} \eqdef (I - \mathcal{C})^{-1}
= \sum_{k=0}^\infty C^k$. Thus, $\mathcal{M} \succeq 0$ since $\mathcal{C}$ is
non-negative (which follows from the fact that $\MarkovP \succ 0$
and $J_i \succeq 0$'s). It follows from the expression for $\psi$ that it is also
non-negative vector. This shows that $q = (I - \mathcal{C})^{-1}\psi
\succeq 0$, which implies $\mu \succeq 0$.
\end{proof}

\medskip

%%%%%%%%%%%%%%%%% proof of the theorem %%%%%%%%%%%%%%%%%%%
Now we are ready to prove Theorem~\ref{thm:mobile-convergence}
\begin{proof}[Proof of Theorem~\ref{thm:mobile-convergence}]
It follows from Proposition~\ref{prop:rho_of_CD} that under the hypothesis of the
Markov chain being ergodic, we have $\rho(\mathcal{D})<1$. It then
follows from Lemma~\ref{lem:ss-var} that the state converges in the
mean square sense, which proves the first statement of the
theorem. Note that the limiting mean and correlation of the state is
also provided by Lemma~\ref{lem:ss-var}.

We already know from Lemma~\ref{lem:ss-var} that $\mu \succeq 0$. To
prove the last statement of the theorem, that $\mu(u)>0$ if and only
if there is a path between node $u$ and the source node $1$ in the
union graph, we have to look at the
structure of the tall vector $q$ in~\eqref{eq:mu-Q} more carefully, since $q$
completely determines $\mu$. With some abuse of notation, from now on the source
node will be referred to as node $1$ instead of $v_1$. Note that $\pi\succ\!\succ0$ which follows from ergodicity,
$\MarkovP\succ\!\succ 0$ by assumption, $B_i$ is a diagonal matrix
with positive diagonal entries for every $i$ (follows from its definition), and $w=se_1$, where $s>0$ and $e_1=[1,0,\dots,0]^T\in \R^n$. It is easy to
show now that the $\psi_j=a e_1$ for some $a>0$. Thus,
$\psi=a[{e_1}^T,\dots,{e_1}^T]^T\in\R^{Nn}$. Since
$\rho{\mathcal{(C)}}<1$ (Proposition \ref{prop:rho_of_CD}),
$\mathcal{M} \eqdef
(I-\mathcal{C})^{-1}=\sum\limits_{k=0}^\infty\mathcal{C}^k$. Now,
we express the matrix $\mathcal{M}$ in terms of its blocks:
$\mathcal{M}=[\mathcal{M}^{(ij)}]$, where $\mathcal{M}^{(ij)}$ are $n
\times n$ matrices. Then, $q$ can be rewritten as,
\begin{align*}
  q =
\begin{bmatrix}
	 q^{(1)}\\
	 q^{(2)} \\
    \vdots \\
	 q^{(N)}\\
\end{bmatrix}=
\begin{bmatrix}
\mathcal{M}^{(11)} & \mathcal{M}^{(12)} & \dots & \mathcal{M}^{(1N)} \\
\mathcal{M}^{(21)} & \mathcal{M}^{(22)} & \dots & \mathcal{M}^{(2N)} \\
\vdots & \vdots & \dots & \vdots \\
\mathcal{M}^{(N1)} & \mathcal{M}^{(N2)} & \dots & \mathcal{M}^{(NN)}
\end{bmatrix}
\begin{bmatrix}
	 ae_1\\
	 ae_1 \\
    \vdots \\
	 ae_1\\
\end{bmatrix}=a[\mathcal{M}^{(ij)}e_1], \quad (a>0)
\end{align*}

Therefore,
$q^{(i)}=\sum_{j=1}^{N}\mathcal{M}^{(ij)}e_1=\sum_{j=1}^{N}\mathcal{M}^{(ij)}_{:1}$,
where the subscript $:1$ denotes the first column of the corresponding
matrix.  Hence, the $u$-th entry of $q^{(i)}$ is
$q^{(i)}(u)=\sum_{j=1}^{N}\mathcal{M}^{(ij)}_{u1}$. Recall that $\mu
=\sum_{i=1}^{N}q^{(i)}$. Therefore $\mu(u)=0$ if and only if
$q^{(i)}(u)=0$ for $i=1,\dots,N$, which is also equivalent to
$\sum_{i=1}^{N}\sum_{j=1}^{N}\mathcal{M}^{(ij)}_{u1}=0$.

The subsequent discussion requires introducing directed graphs
associated with matrices. For every $\ell
\times \ell$ matrix $A$, define $\vec{\G}(A) = (\V,\vec{\E})$
be the directed graph corresponding to $A$ as follows: the node set
$V$ is the index set $\V = \{1,\dots,\ell\}$ and the edge set
is defined by $(i,j)\in \vec{\E}$ if and only if $A_{i,j} \neq 0$~\cite{CarlMeyer:01}.
It is a standard result in graph theory that the number of walks from
a vertex $i$ to vertex $j$ in a directed graph of length $r$ is the $(i,j)$-th
element of $A^r$, where $A$ is the adjacency matrix of the
graph~\cite[pp.~165]{GodsilRoyle_2001}. Since $\mathcal{M} =
\sum\limits_{k=0}^\infty\mathcal{C}^k$, it follows from the preceding discussion that the
$(i,j)$-th entry of $\mathcal{M}$ is positive if and only if there exists a path
from the vertex $i$ to vertex $j$ in the directed graph
$\vec{\G}{(\mathcal{C})}$. Note that the graph $\vec{\G}{(\mathcal{C})}$ contains $Nn$ nodes.  We can group $Nn$ nodes into $N$
clusters such that each cluster, containing $n$ nodes, can be thought
of as copies of the $n$ nodes in the sensor and robot network. To
prevent confusion between the vertices in
$\vec{\G}{(\mathcal{C})}$ and node set $\mathcal{V}$ of the
original network, we use $v^{(i)}$ to denote a node in the graph
$\vec{\G}{(\mathcal{C})}$  that is the $i$-th copy of the
node $v$ in $\mathcal{V}$, where $i=1,\dots,N$.

Therefore $\sum_{i=1}^{N}\sum_{j=1}^{N}\mathcal{M}^{(ij)}_{u1}=0$ is
equivalent to there being no directed path from any of the $u$'s
copies ($u^{(i)}, i=1,\dots,N$) to any of $1$'s copies ($1^{(i)},
i=1,\dots,N$) in the directed graph
$\vec{\G}(\mathcal{C})$. Otherwise, $q(u)>0$. Since existence of an
edge from $i$ to $j$ in $\vec{\G}(A)$ only
depends on whether the $i,j$-th entry of $A$ is non-zero, and does not
depend on the specific value of the entry, it is convenient to define
$\overline{A}$ be a matrix associated with the matrix $A$, such that
$\overline{A}_{ij}=1$ if $A_{ij} \neq 0$ and $\overline{A}_{ij}=0$ if $A_{ij} =0$. Since $\MarkovP \succ\!\succ 0$, we have
\begin{align*}
  \overline{\mathcal{C}} = [\overline{p_{ij}J_j}]=
\begin{bmatrix}
\overline{J_1} & \overline{J_2} & \cdots & \overline{J_N} \\
\overline{J_1} & \overline{J_2} & \cdots & \overline{J_N} \\
\vdots & \vdots & \cdots & \vdots \\
\overline{J_1} & \overline{J_2} & \cdots & \overline{J_N}
\end{bmatrix}_{Nn}.
\end{align*}
It can be seen in a straightforward manner upon examining the matrix $\overline{\mathcal{C}}$ that if
there is an edge between nodes $u$ and $v$ in the $i$-th graph $\G_i$,
i.e., $(u,v) \in \E^{(i)}$, then $(u^{(j)},v^{(i)})$ for all
$j=1,\dots,N$, and $(v^{(j)},u^{(i)})$ for all $j=1,\dots,N$, i.e.,
there are edges in $\vec{\G}(\mathcal{C})$ from all copies of $u$ to $v^{(i)}$,
the $i$-th copy of $v$, and from all copies of $v$ to $u^{(i)}$, the
$i$-th copy of $u$.

Now we will show that if an arbitrary node $u$ is connected to $1$ in
the union graph $\cup_{i=1}^{N}\G_i$, then there is a path from a copy
of $u$ to a copy of $1$ in the directed graph $\vec{\G}(\mathcal{C})$,
otherwise not. To see that this is the case, we first take an example:
consider a path of length $2$ from $u$ to $1$ in the union graph that
involves two edges in two distinct graphs: $(u,v) \in \E^{(2)}$ and
$(v,1) \in \E^{(1)}$. From the preceding discussion, we have that
$(v,1) \in \E^{(1)} \Rightarrow (v^{(1)},1^{(1)}), (v^{(2)},1^{(1)})
\in \vec{\E}(\mathcal{C})$, and $(u,v) \in \E^{(2)} \Rightarrow
(u^{(1)},v^{(2)}), (u^{(2)},v^{(2)}) \in \vec{\E}(\mathcal{C})$. Thus
a path from a copy of $u$ to a copy of $1$ in $\vec{\G}(\mathcal{C})$
is $p = \{ (u^{(1)},v^{(2)}),(v^{(2)},1^{(1)})\}$. This argument works
as long as there is a path from $u$ to $1$ in the union graph,
irrespective of how long the path is. This shows that $u$ is connected
to $1$ in the union graph, then there is a path from at least one of
its copies to one of $1$'s copies in the directed graph
$\vec{\G}(\mathcal{C})$, which means $q(u)>0$. If, however, $u$ is not
connected to $1$ in the union graph, we can show that there is no path
from any of $u$'s copies to any of $1$'s copies. This can be shown by
considering the set of all nodes that do not have paths to $1$ in the
union graph and the set of nodes that do separately;
see~\cite{CL_mobileDSSD_techreport:11} for details. This concludes the
proof of the last statement of the theorem.
\end{proof}

\section{Simulation Tests}\label{SECTION_Simulations}
The DSSD algorithm was tested in a $\text{MATLAB}^{\mathrm{TM}}$
simulation for a network consisting of $200$ agents initially deployed
in a unit square at random. Two agents can only establish direct
communication if their Euclidean distance is less than $0.11$. The
source strength and cut detection threshold was $s=5\times10^{5}$ and
$\epsilon=10^{-2}$, respectively. Since there is no existing prior
work on the problem of detecting separation in mobile networks that
can operate without multi-hop routing, we do not provide simulation
comparison with existing algorithms. Note that the solutions proposed
in~\cite{Hauspie03partitiondetection,HR_RW_JS_TZ_SIGCOM:04} require
routing between the nodes and the base station, which is challenging
in sensor and robotic networks in which the topology can change with
time quickly.

\subsection{Performance of DSSD in a static network}
The first set of simulations is conducted with $200$ static nodes (see
Figure~\ref{fig:graph4sim}(a)). The center node (symbolized by a triangle) is
the source node. Simulations are run in a synchronous manner and a neighbor
is removed from the list of neighbors of a node the first time it failed to
receive messages from that neighbor. At $k=100$ the nodes shown as red squares
in Figure~\ref{fig:graph4sim}(b) fail, leading to a cut in the network. Figure~\ref{fig:graph4sim}(c-d) show the time evolution of the states
(calculated using \eqref{eq:stateupdatelaw}) of the four nodes $u$, $v$, $w$,
and $z$. Node $v$ is the only one among the four that is separated from the
source after the cut occurs. Initially, the states of every node increase from
$0$ and then settle down to their steady state value. After the cut occurs,
the state of node $v$ decreases towards $0$. When the state of node $v$
decreases below the preset threshold $\epsilon$, it declares itself cut from
the source. This occurs at $k=133$, thus the delay between the occurrence of
the cut and its detection by $v$ is $33$ time-steps.

\begin{figure}[ptb]
\begin{center}
\psfrag{u}[][][1.5]{$u$} \psfrag{v}[][][1.5]{$v$}
\psfrag{xuk}[][][1]{states}
% \psfrag{xuk}[][][1]{$x_u(k)$}
% \psfrag{xvk}[][][1]{$x_v(k)$}
\psfrag{iterindex}{$k$ (iter.~index)}
\subfigure[ $\G$ before cut]{\includegraphics[width=0.27\textwidth]{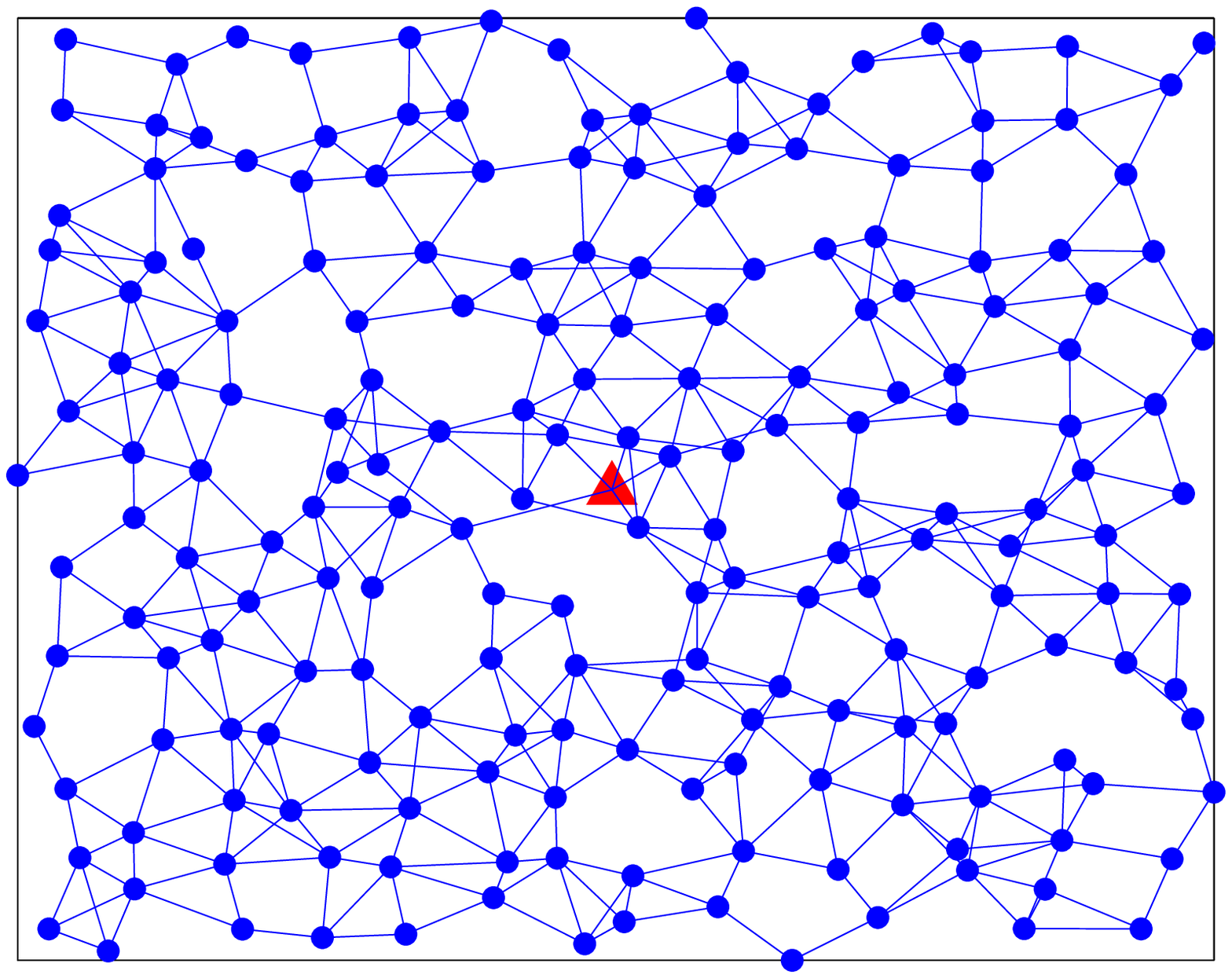}}
\subfigure[ $\G(k)$ for $k > 100$ ]{\includegraphics[width=0.33\textwidth]{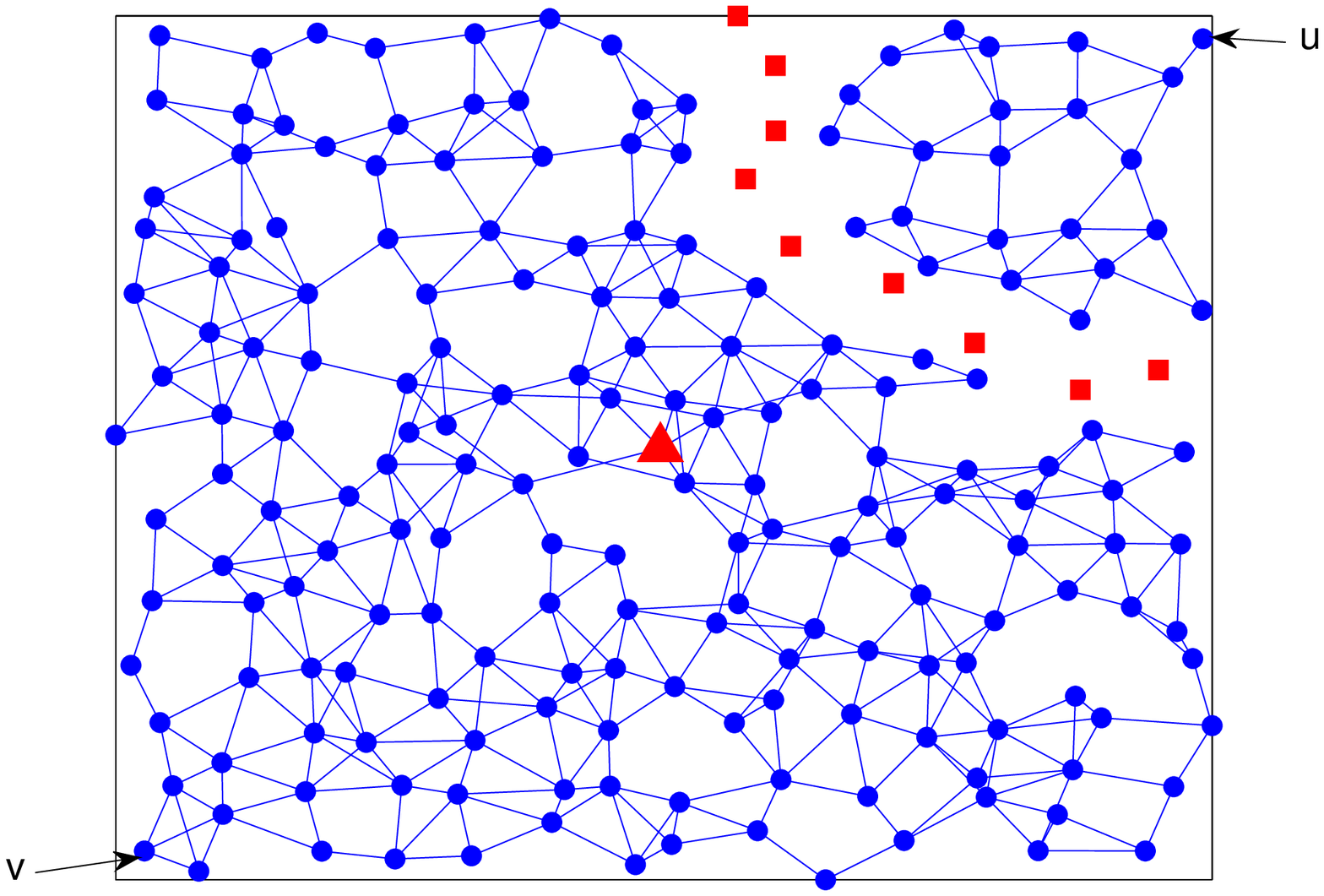}}
\subfigure[$x_u(k)$ and $x_v(k)$ vs. $k$]{\includegraphics[width=0.30\textwidth]{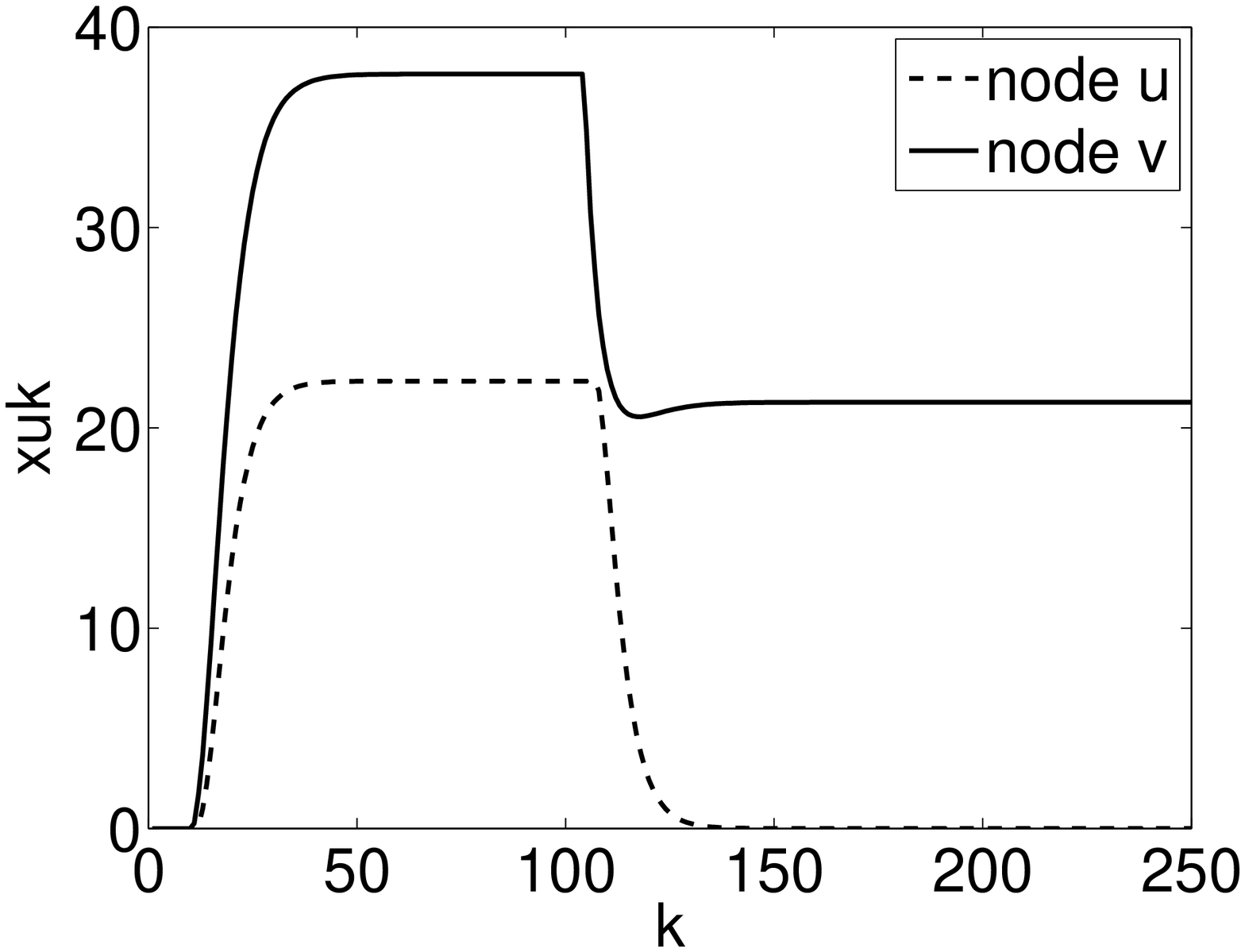}}
%\subfigure[ state of agent $v$ ]{\includegraphics[scale=0.325]{./figures/staticsim_vhist.eps}}
\end{center}
\caption{{\small {(a)-(b): A sensor network with $200$ static nodes, shown
before and after a cut occurs due to the failure of the nodes shown as red
squares. The cut occurs at $k=100$. (c): The states of two nodes $u$ and $v$
as a function of iteration number. The source node is at the center
(triangle), and the source strength is chosen as $s= 5 \times10^{4}$.}}}%
\label{fig:graph4sim}%
\end{figure}

\subsection{Performance of DSSD in a mobile network}
Figures~\ref{fig:mobile_graph4sim}(a-d) show four snapshots of a communication
network of $200$ mobile agents. The agents are divided into two groups, though there is no clear spatial separation between the two groups initially. The position of agent $u$, denoted by $Z_{u}$, is updated according to:
\begin{equation}
Z_{u}(k+1)=Z_{u}(k)+\left[
\begin{array}
[c]{c}%
\delta Z_{ux}(k)\\
\delta Z_{uy}(k)
\end{array}
\right]
\end{equation}
where $\delta Z_{ux}(k)$, $\delta Z_{uy}(k)$, for every $u$ and $k$, are
independent random numbers. For agents in the first group, both $\delta Z_{ux}$
and $\delta Z_{uy}$ are normally distributed with mean $0.003$ and variance
$0.0003$. For the second group, $\delta Z_{ux}$, $\delta Z_{uy}$ are normally
distributed with mean $-0.003$ and variance $0.0003$. The motion of the agents
results in the network composed of two disjoint components at $k=28$, four
components at $k=56$, and then again two components at $k=80$.

\begin{figure}[t]
\begin{center}
\psfrag{i}{$i$} \psfrag{j}{$j$} \psfrag{p}{$p$} \psfrag{q}{$q$}
\psfrag{iterindex}{$k$ (iter.~index)}
\hspace{1 cm}\subfigure[ $\G$ at $k= 24$ ]{\includegraphics[scale=0.3]{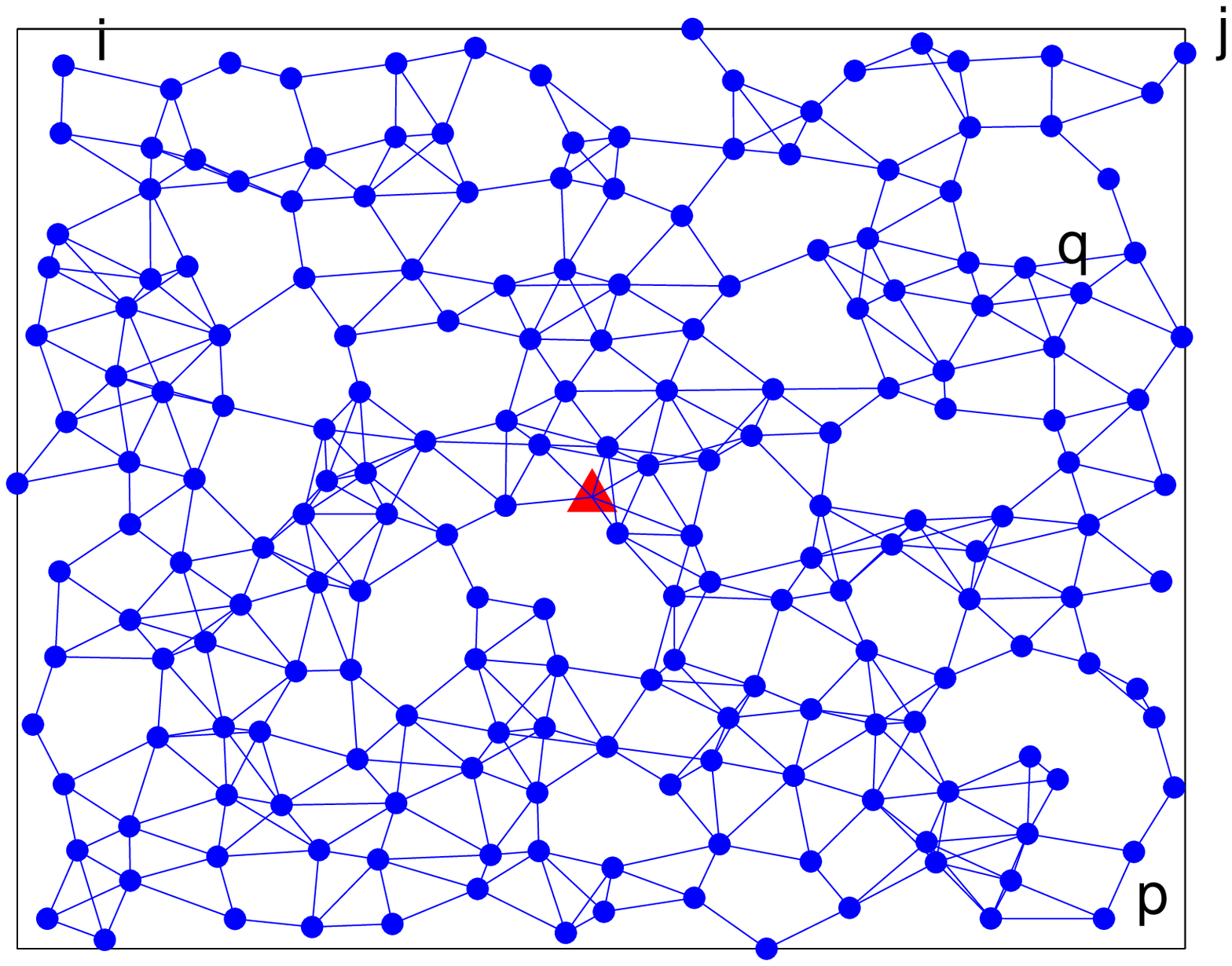}}\hspace
{8 ex}
\subfigure[ $\G(k)$ for $k = 29 $ ]{\includegraphics[scale=0.3]{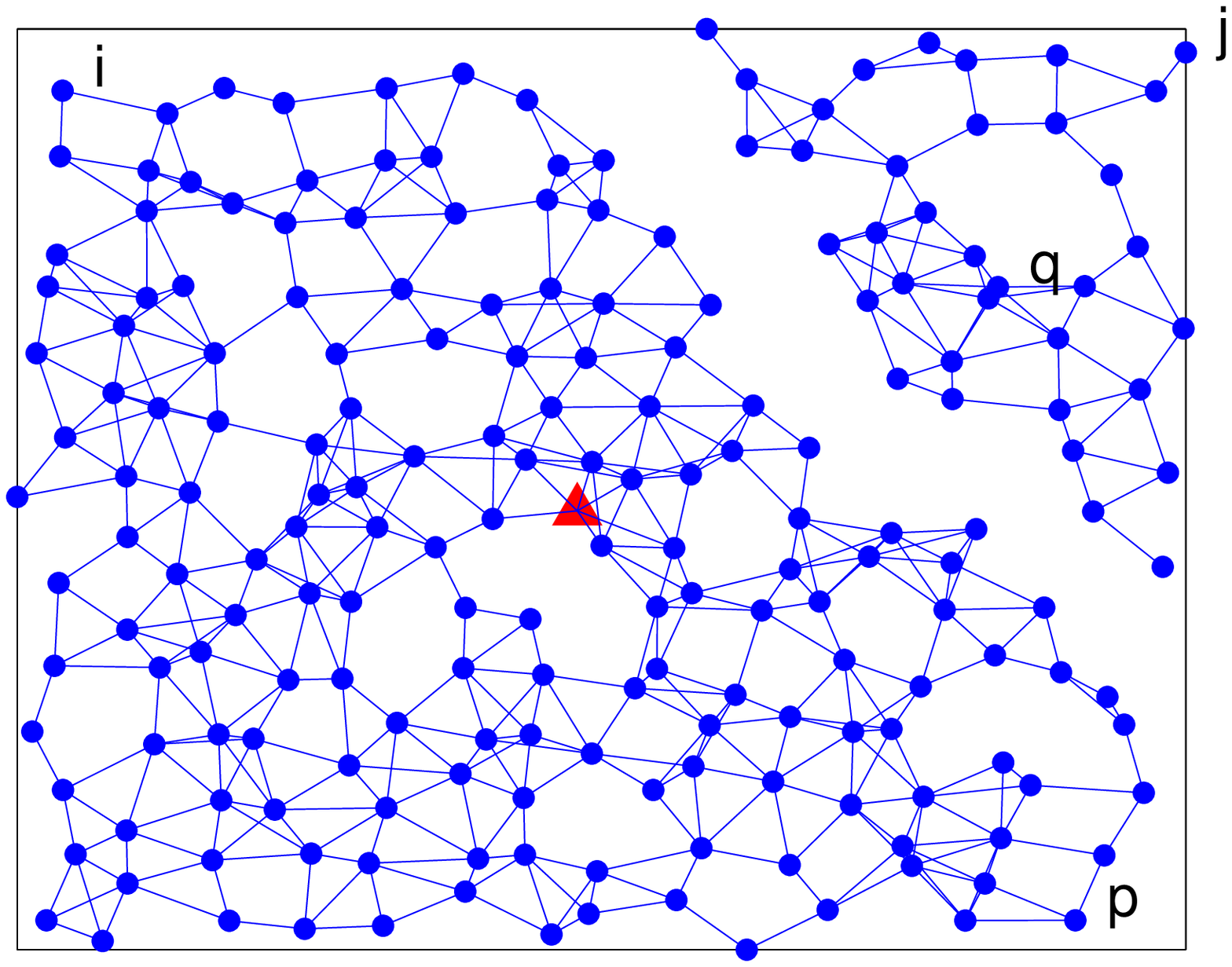}}\newline%
\vspace{- 2ex}
\hspace{1 cm}\subfigure[ $\G(k)$ for $k = 56 $ ]{\includegraphics[scale=0.3]{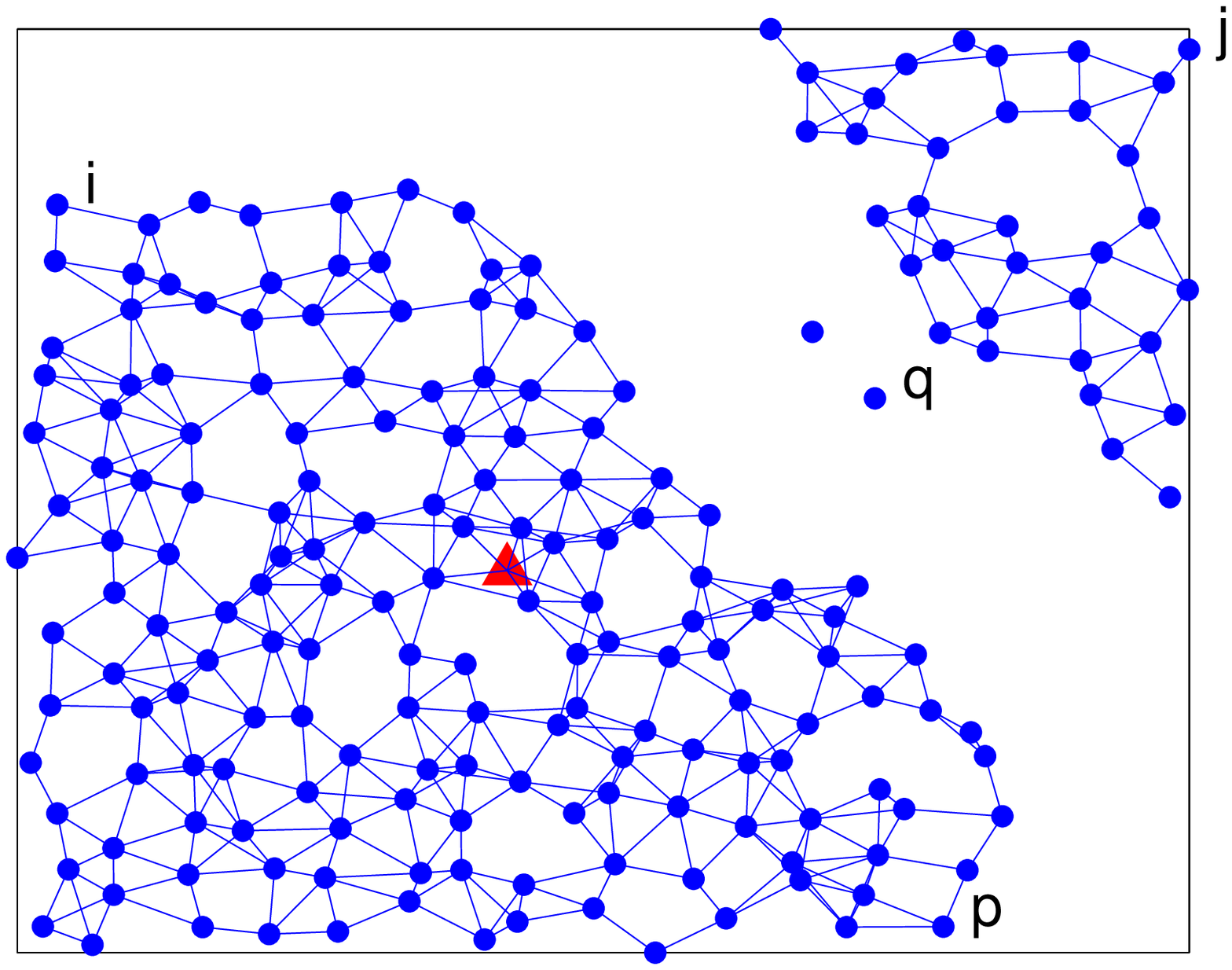}}\hspace
{8 ex}
\subfigure[ $\G(k)$ for $k = 99 $ ]{\includegraphics[scale=0.3]{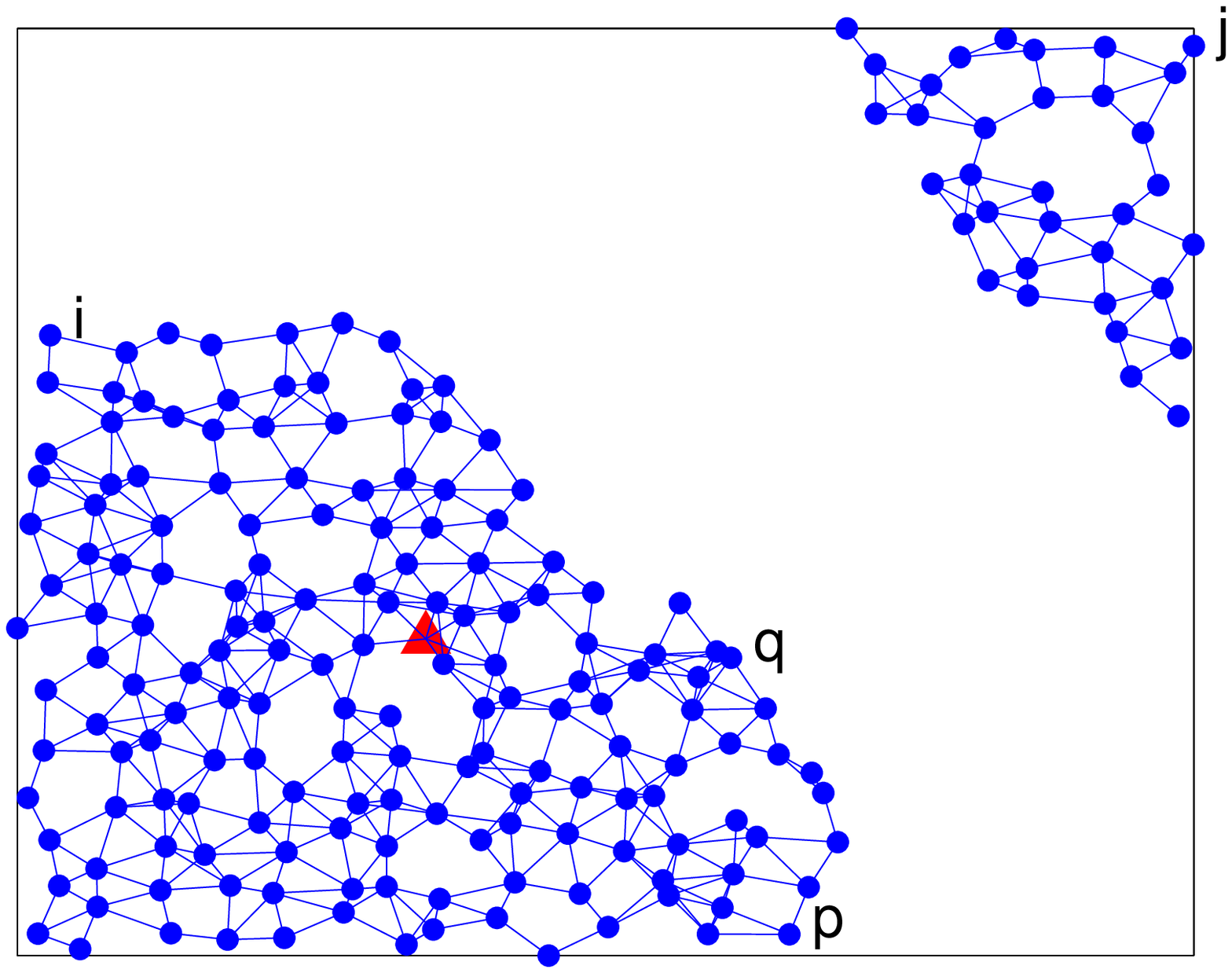}}\newline
\end{center}
\caption{Four snapshots of a network of $200$ mobile agents.}%
\label{fig:mobile_graph4sim}%
\end{figure}

The evolution of the states of four agents $i$, $j$, $p$, and $q$ are
shown in Figure~\ref{fig:mobile_sim_results}(a-b). The loss of connectivity of agent
$q$ from the source occurs at $k=28$ and is detected at $k=55$. Connectivity
to the source is regained at $k=80$ and is detected at $k=81$ (when the
states became greater than $\epsilon$).
\begin{figure}
\centering
\subfigure[states of node $i$ and $j$ ]{\includegraphics[scale=0.3]{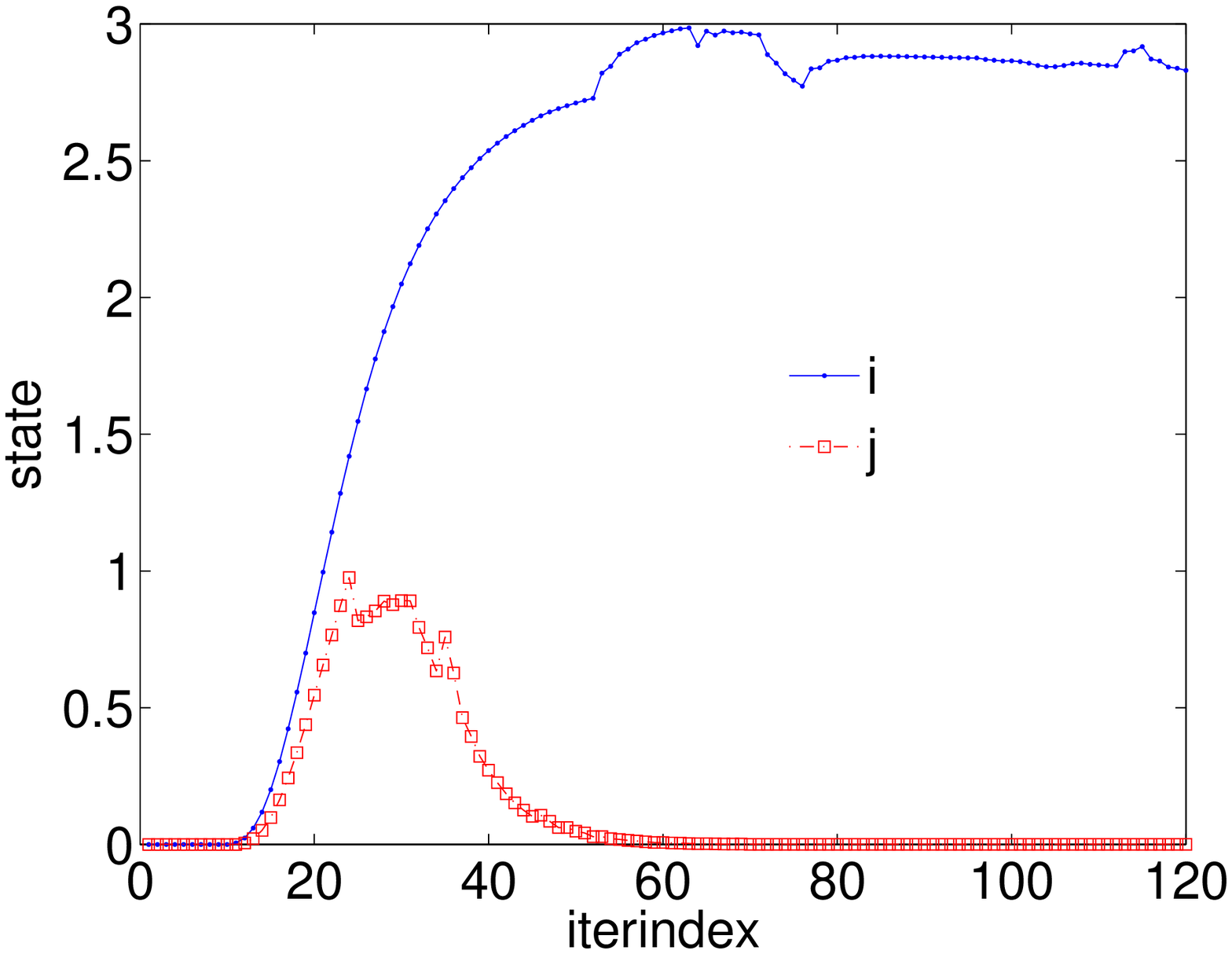}}
\subfigure[states of node $p$ and $q$ ]{\includegraphics[scale=0.3]{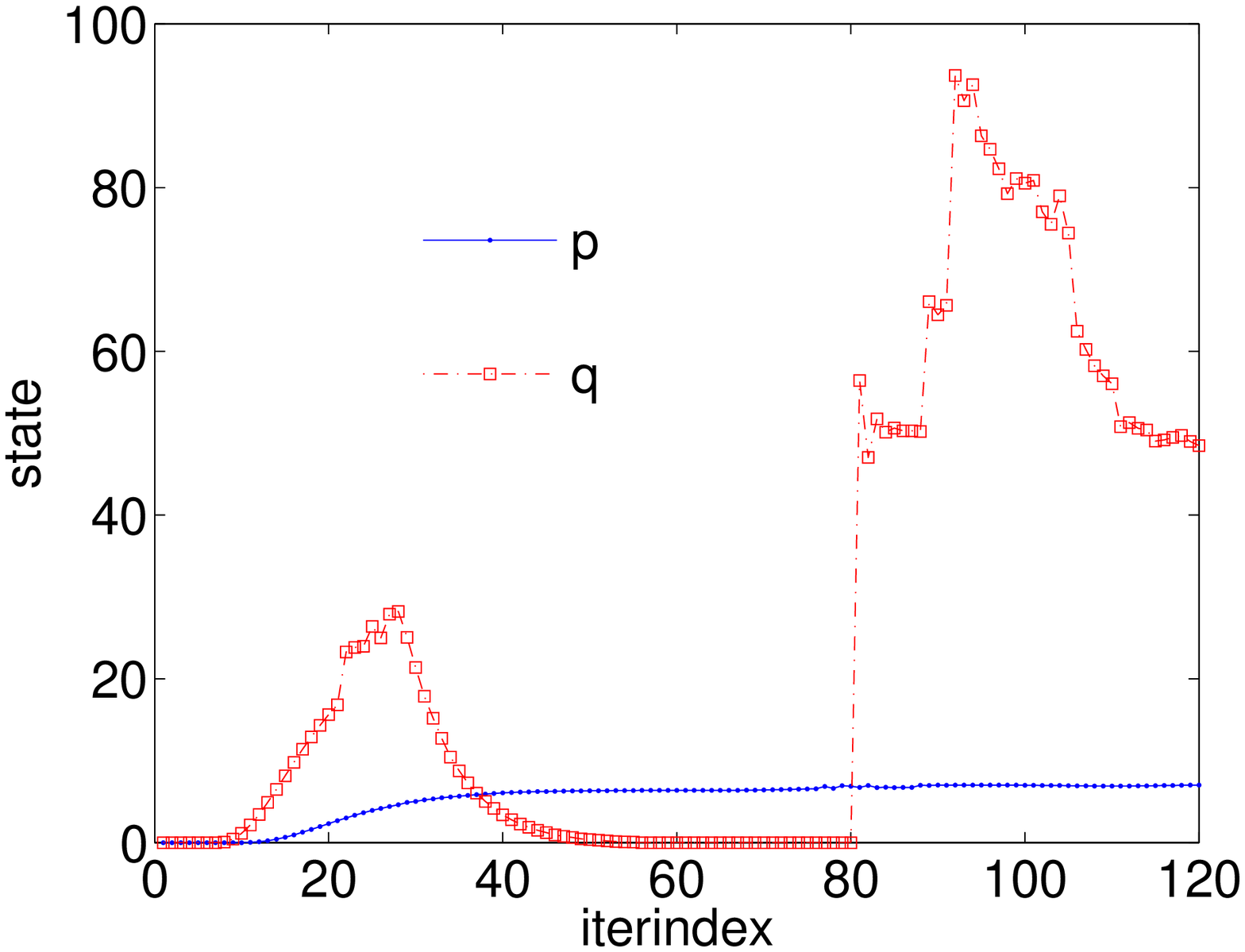}}
\caption{The states of four mobile nodes $i,j,p,q$ (as a function of time) in the network
  shown in Figure~\ref{fig:mobile_graph4sim}. }%
\label{fig:mobile_sim_results}%
\end{figure}
%%% plot of state time history of a few representative nodes
%\begin{figure}[tbh]
%\psfrag{u}{node $u$} \psfrag{v}{node $v$} \psfrag{z}{node $z$}
%\psfrag{p}{node $p$} \psfrag{xuk}{$x_u(k)$}
%\psfrag{iterindex}{$k $ (iter.~index)}
%\par
%\begin{center}
%\includegraphics[scale = 0.6]{./mobilegraph_4nodeshist.eps}
%\end{center}
%\caption{The states of four mobile agents $u, v, z$ and $p$ (see
%Figure~\ref{fig:mobile_graph4sim} for their locations) as a function of
%iteration number. The loss and regaining of connection between node $p$ and
%the source is evident from the plot of agent $p$'s state. }%
%\label{fig:nodestatehistory_mobile}%
%\end{figure}
These simulations provide evidence that the algorithm is indeed
effective in detecting disconnections and re-connections, irrespective
of whether the network is made up of static or mobile agents.

\section{System Implementation and Experimental Evaluation}\label{sec:system_impl}

In this section we describe the implementation, deployment and performance evaluation of a separation detection system for robotic sensor networks based on the DSSD algorithm. We implemented the system, using the nesC language, on Berkeley motes~\cite{crossbow03manual} running the TinyOS operating system~\cite{hill00tinyos}. The code uses 16KB of program memory and 719B of RAM. The separation detection system executes in two phases: Reliable Neighbor Discovery, and the DSSD algorithm.

\begin{figure}[t]
\centering
\includegraphics[width=.6\textwidth]{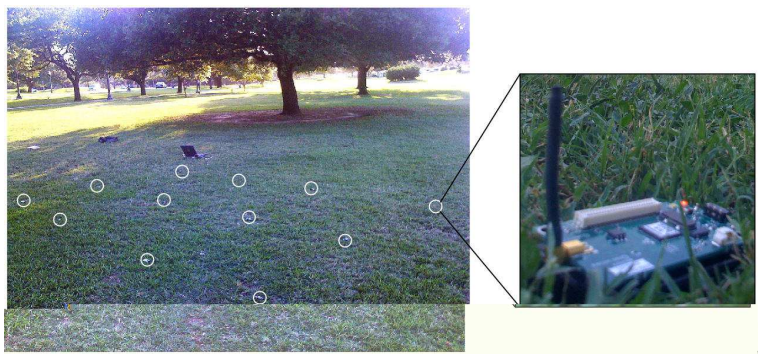}
\caption{Partial view of the 24 node outdoor system deployment.}
\label{fig:deployment}
\end{figure}

\begin{figure}[t]
\begin{center}
\subfigure[The graph]{\includegraphics[scale = 0.6]{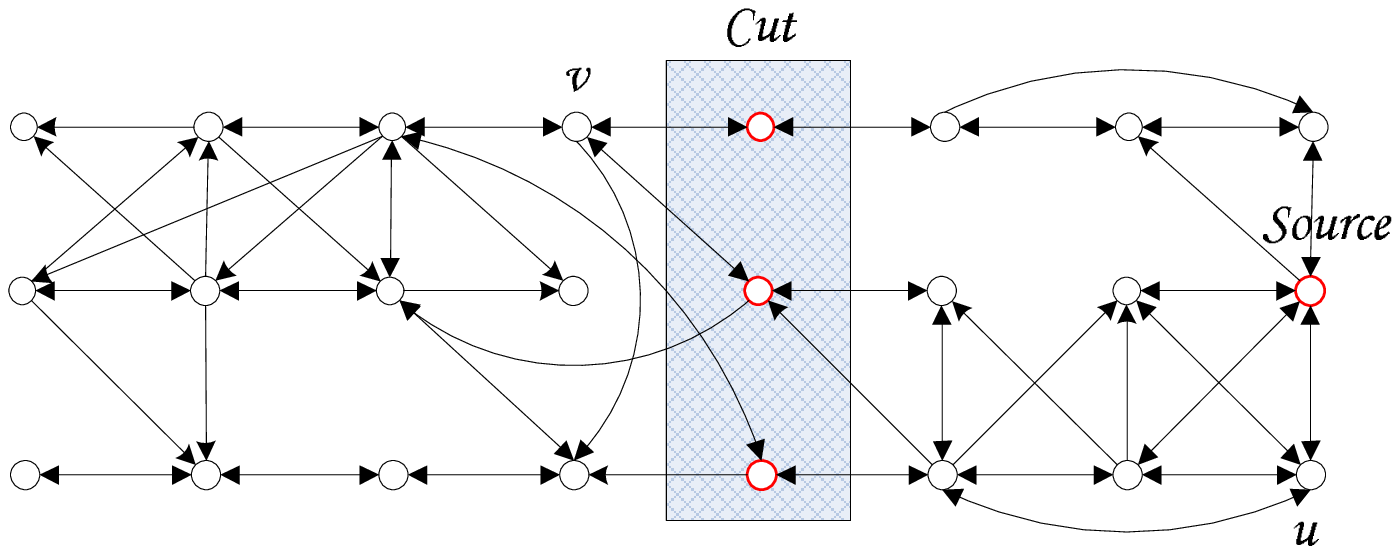}}
%\subfigure[The graph]{\includegraphics[scale = 0.65, angle = 90]{./figures/bar.eps}}
\subfigure[State histories]{\includegraphics[scale = 0.23]{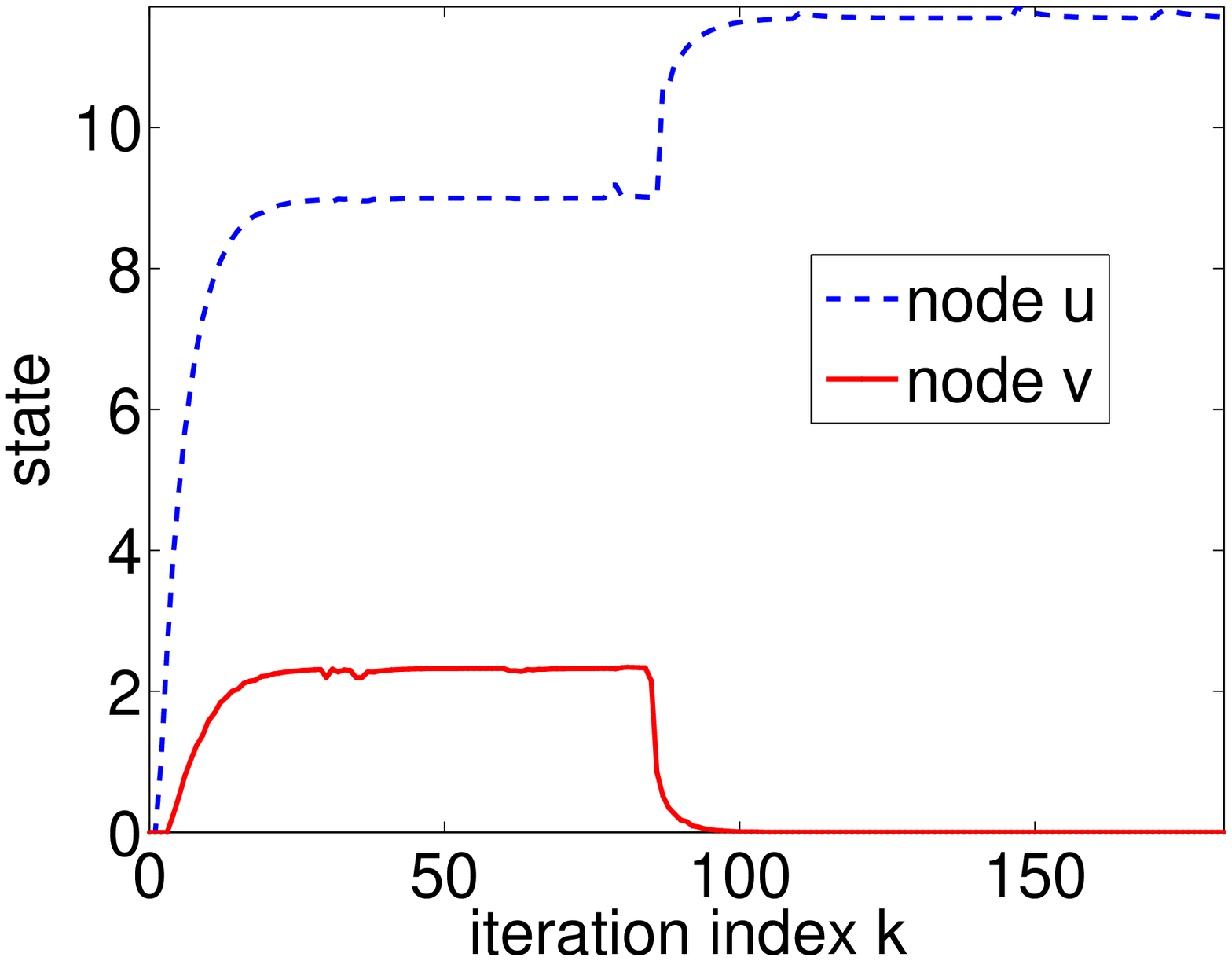}}
\end{center}
\caption{(a) The network topology during the outdoor deployment. (b) The states of nodes $u$ and $v$ (as labeled in (a)), which are connected and disconnected, respectively, from the source after the cut has occurred.}
\label{fig:outdoor_results}
\end{figure}

In the Reliable Neighbor Discovery Phase each node broadcasts a set of beacons
in a small, fixed, time interval. Upon receiving a beacon from node $v_i$, a node updates the number of beacons received from node $v_i$. Next, an iteration of the DSSD algorithm executes. To determine whether a communication link is established, each node first computes for each of its neighbors the Packet Reception Ratio (\textrm{PRR}), defined as the ratio of the number of successfully received beacons received from, to the total number of beacons sent by, a neighbor. A neighbor is deemed reliable if the $\textrm{PRR}>0.8$. After receiving state information from neighbors, a node updates its state according to Equation~\eqref{eq:stateupdatelaw} and broadcasts its new state. When broadcast from a neighbor is not received for 2 iterations, the last reported state of the neighbor is used for calculating the state. A neighbor from which broadcast is not received for 4 iterations is permanently removed from the neighbor table. The state is stored in the 512KB on-board flash memory at each iteration (for a total of about 1.6KB for 200 iterations) for post-deployment analysis. In order to monitor connectivity information each node broadcasts its neighbor table along with the state.

To ensure a lock-step execution of the algorithm, all nodes are started at approximately the same time. For this, a mote acting as a base station, connected to a laptop, broadcasts a \textquotedblleft system start\textquotedblright\ message, which is resent by each sensor node at most once. The base station is also used for monitoring the execution of the algorithm and monitoring the inter-mote communication.

\subsection{Experimental Performance Evaluation in Static Network}\label{sec:exp-static}

\begin{figure}[t]
  \centering
    \subfigure[]{\includegraphics[width=.5\textwidth]{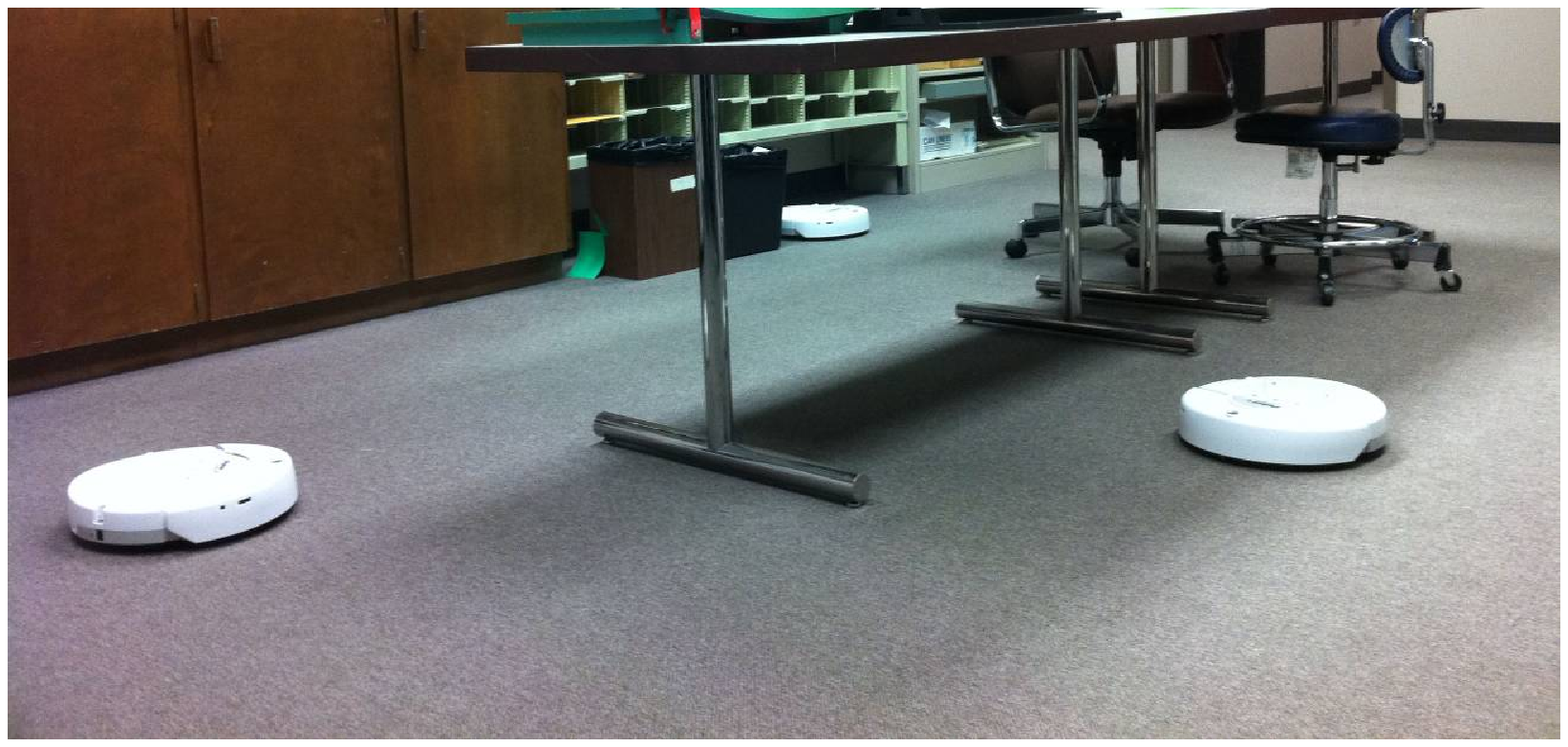}}
    \hspace{3mm}
    \subfigure[]{\includegraphics[width=.15\textwidth]{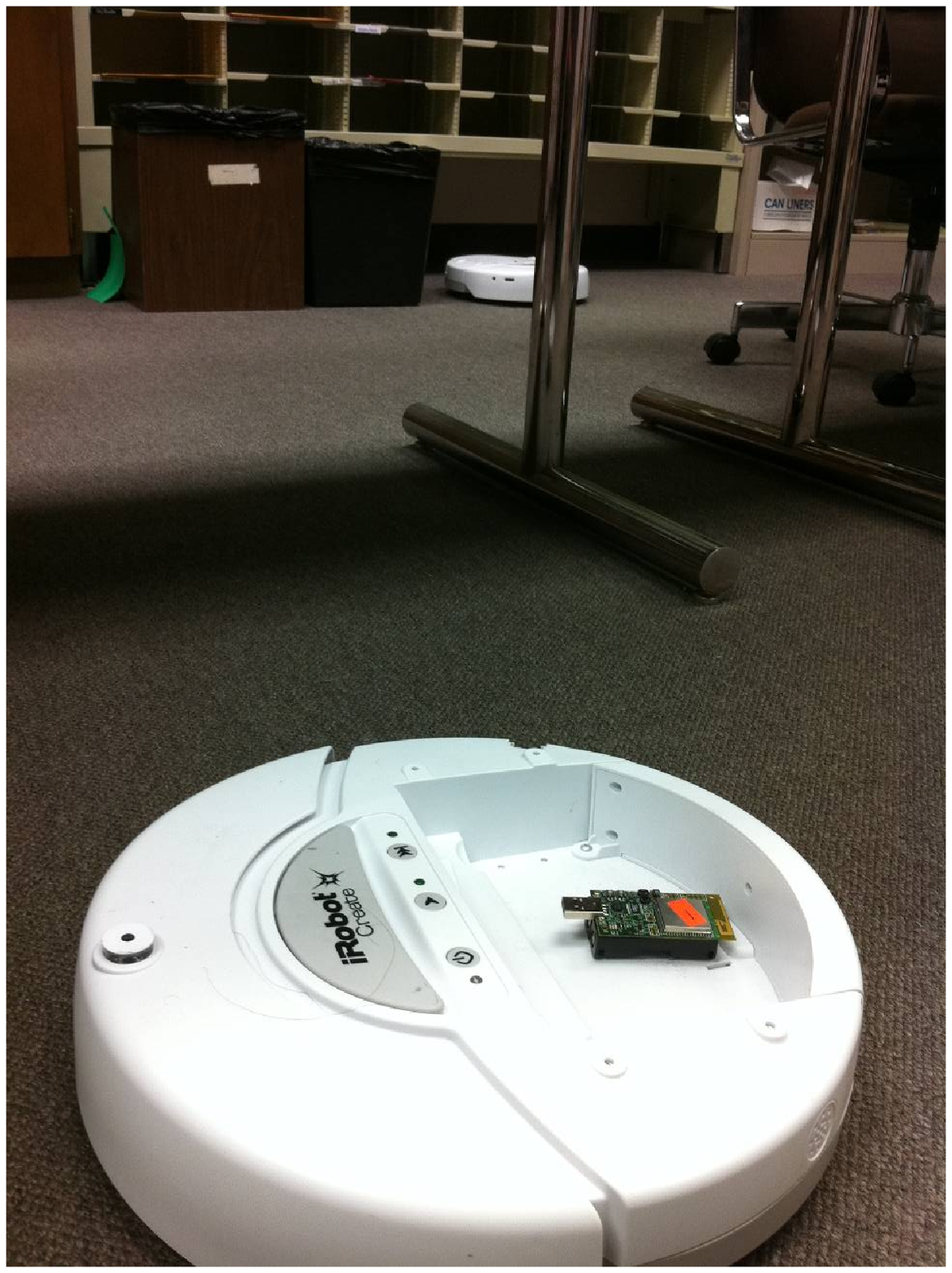}}
  \caption{(a) Test set-up for mobile network experiments. The human agents are not shown. (b) Mobile node consisting of a Berkeley mote on a Roomba robot.}
  \label{fig:mobile-exp-photo}
\end{figure}

\begin{figure}[t]
\begin{center}
\psfrag{(m)}{ft}
\psfrag{i}{$i$} \psfrag{j}{$j$} \psfrag{p}{$p$} \psfrag{q}{$q$}
\psfrag{iterindex}{$k$ (iter.~index)}
\hspace{1 cm}\subfigure[ $\G(k)$ at $k= 110$ ]{\includegraphics[scale=0.3]{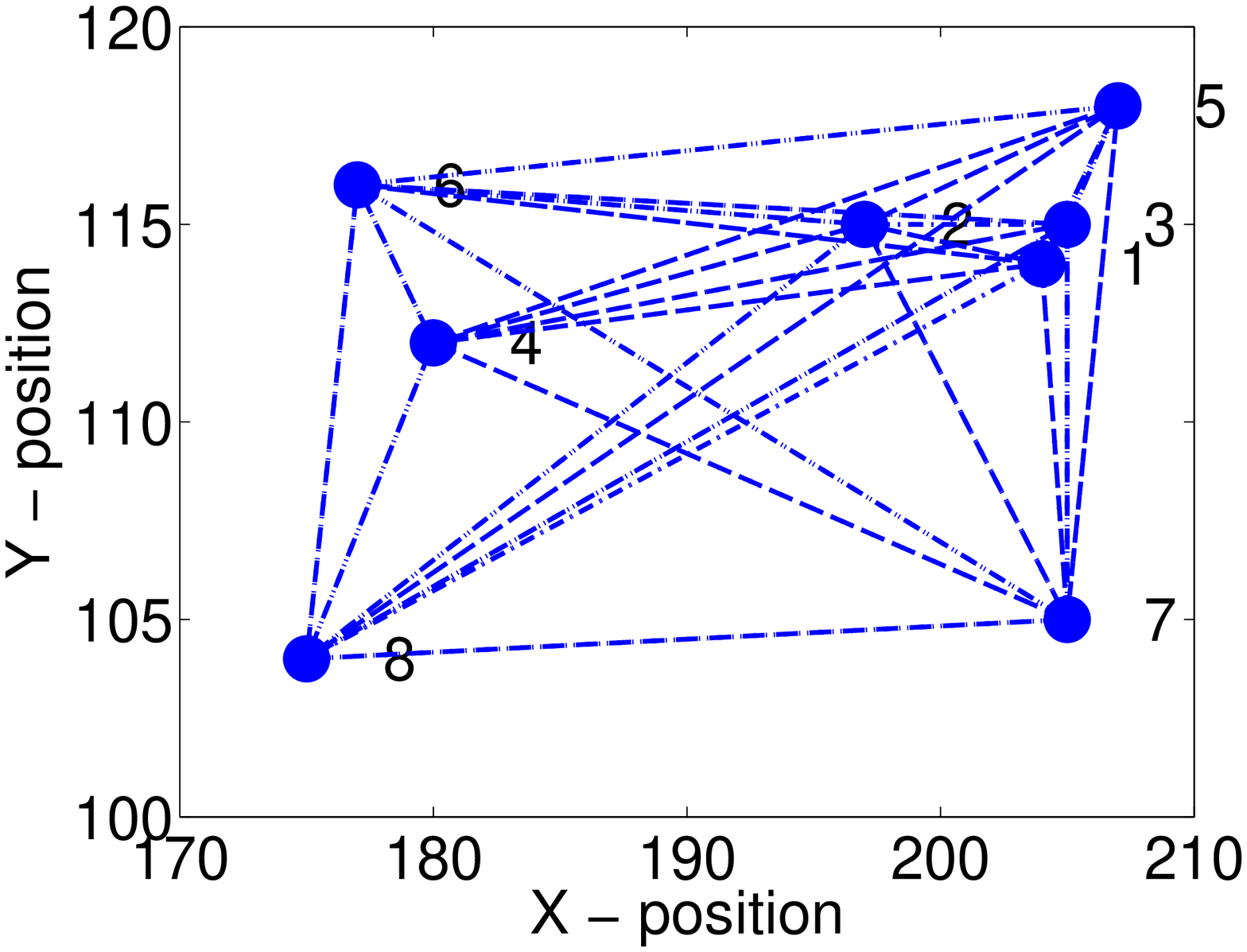}}\hspace
{8 ex}
\subfigure[ $\G(k)$ for $k = 150 $ ]{\includegraphics[scale=0.3]{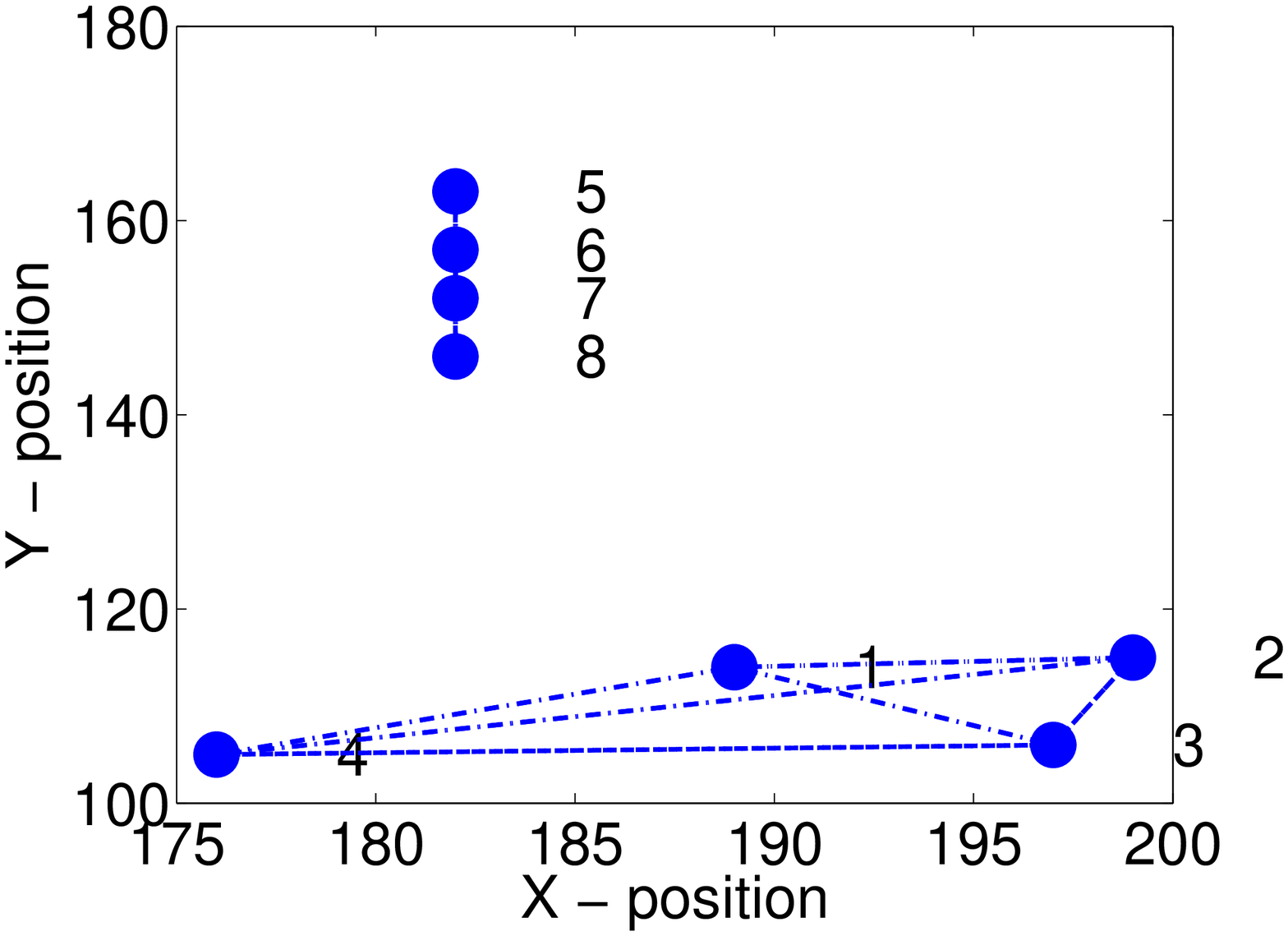}}\newline%
\vspace{- 2ex}
\hspace{1 cm}\subfigure[ $\G(k)$ for $k = 175 $ ]{\includegraphics[scale=0.3]{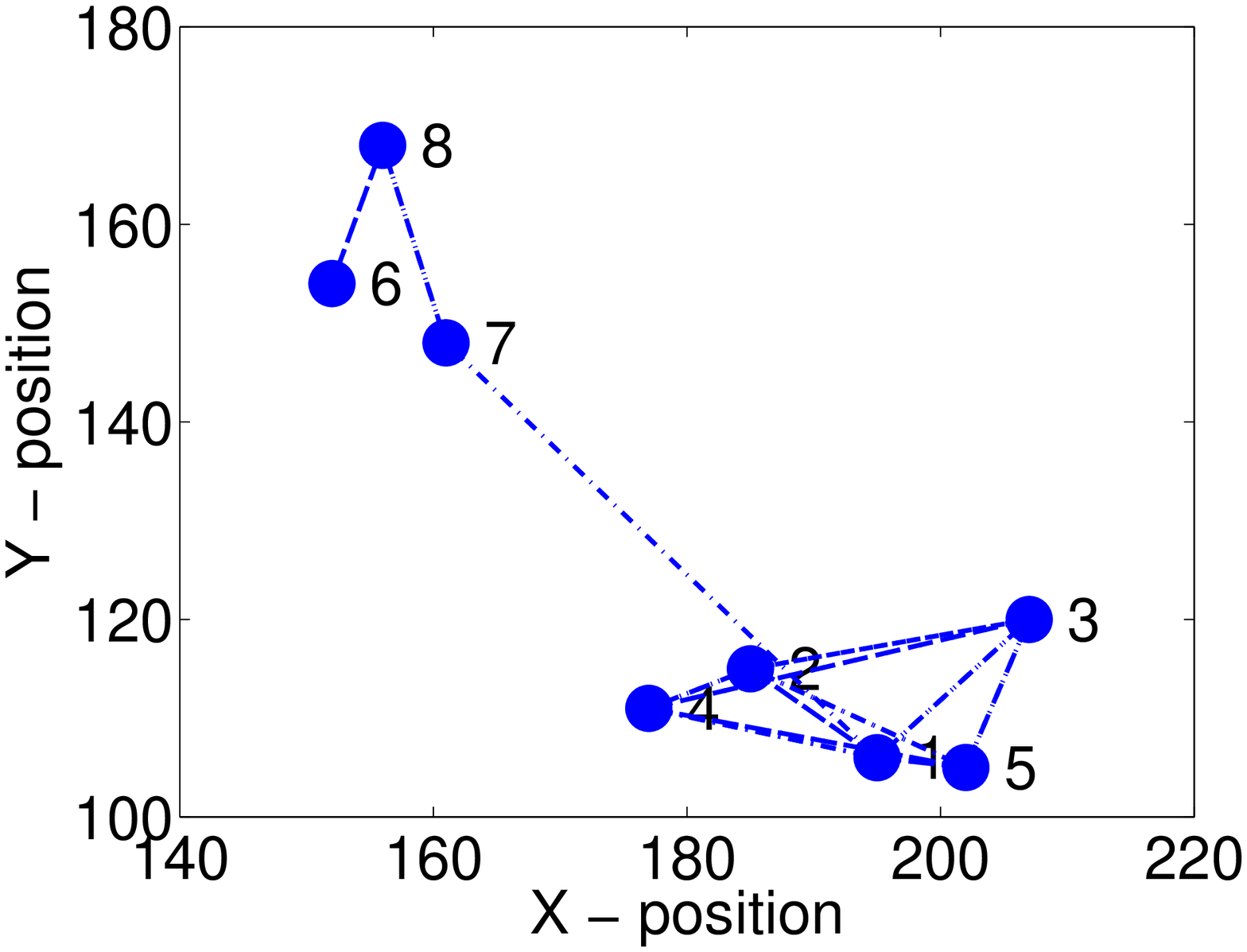}}\hspace
{8 ex}
\subfigure[ $\G(k)$ for $k = 250 $ ]{\includegraphics[scale=0.3]{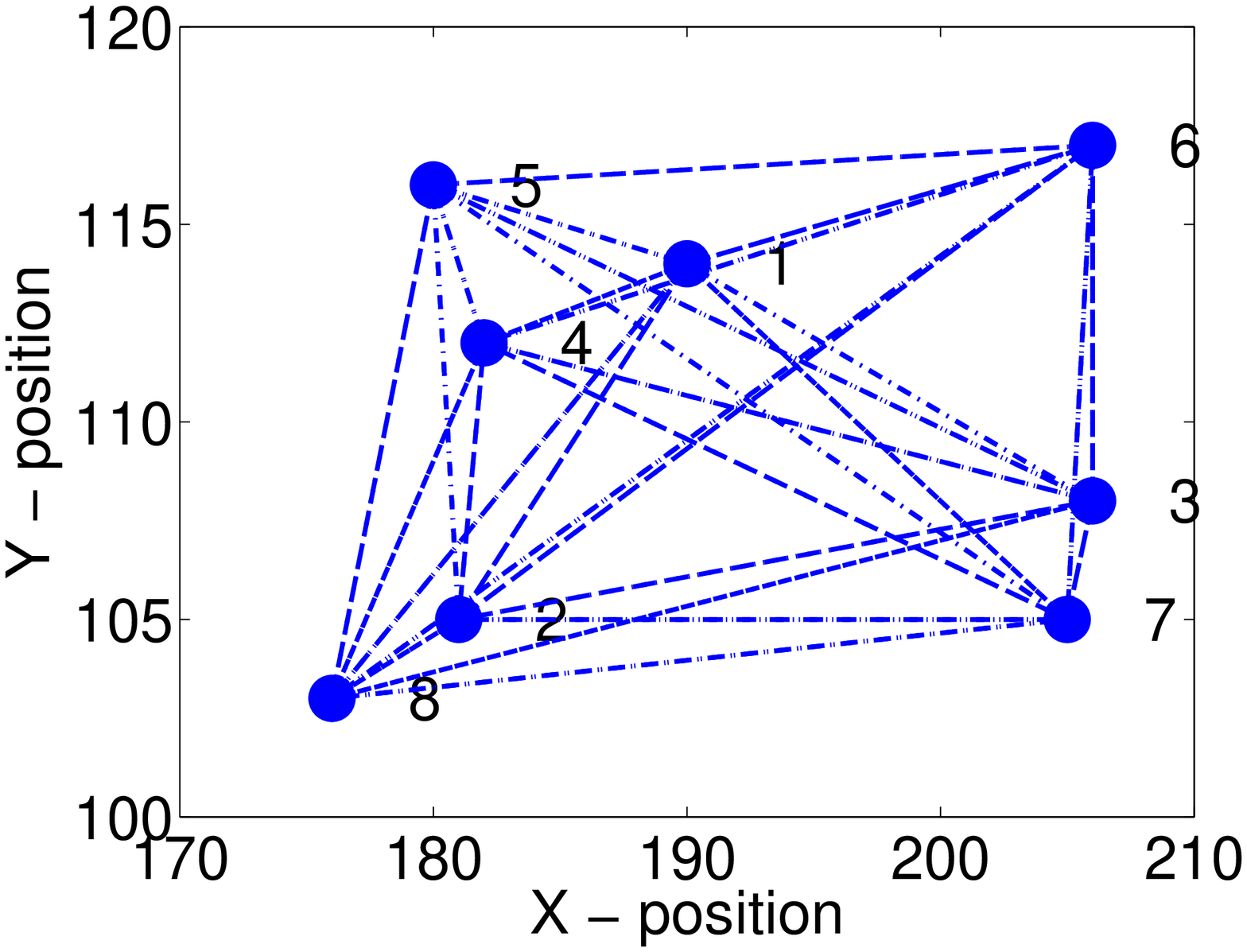}}
\end{center}
\caption{Four snapshots of a network of $8$ mobile agents and the state evolution for them resulting from the DSSD algorithm. The dashed lines represent communication links.}%
\label{fig:mobile_graph_experiment}%
\end{figure}

\begin{figure}[t]
  \centering
\includegraphics[scale=0.4, angle=-90]{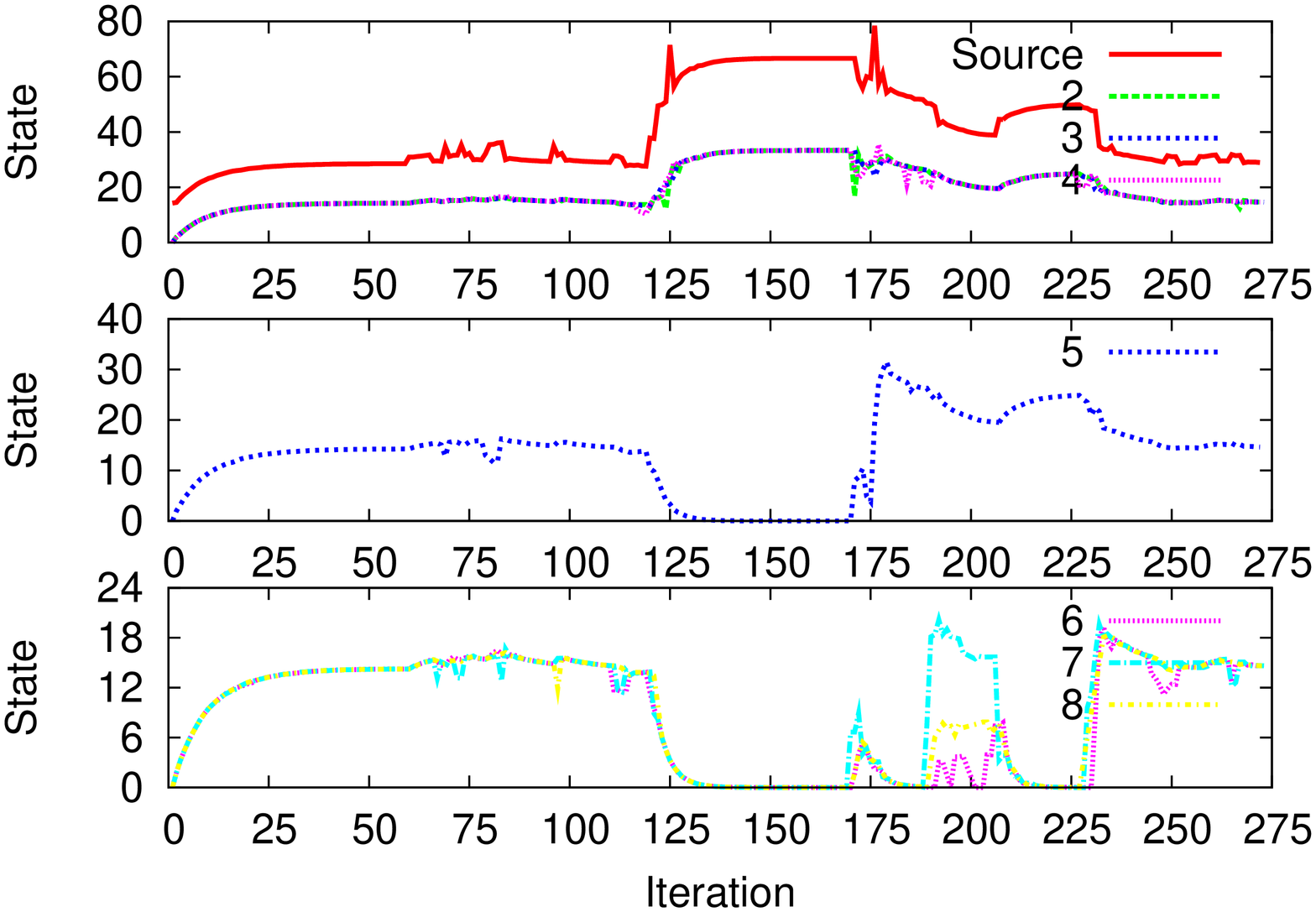}
  \caption{The states of nodes $1$ through $8$ in the mobile network
    experiment with Roomba robots and human agents.}
  \label{fig:states-mobile-exp}
\end{figure}

For evaluating the performance of our separation detection system in
static networks, we deployed a network of 24 motes in a
13$\times$5m$^2$ outdoor field at Texas A\&M University. Because the
motes were positioned on the ground the radio range was reduced
considerably with a one-hop distance of about 1.5m. The network
connectivity is depicted in~Figure~\ref{fig:outdoor_results}(a). A partial view of the outdoor deployment is shown in Figure~\ref{fig:deployment}.

% \begin{figure}[t]
% \begin{center}
% \hspace{.5in} \subfigure[]{\includegraphics[width=.6\textwidth]{./figures/bar.eps}}
% % \hspace{.5in} \subfigure[]{\includegraphics[width=.4\textwidth]{./figures/matlab_plots/static-adjacency.eps}}
% \subfigure[]{\includegraphics[width=.4\textwidth]{./figures/static-state-8.eps}}
% \subfigure[]{\includegraphics[width=.4\textwidth]{./figures/static-state-12.eps}}
% \end{center}
% \caption{(a) This is left for reference. It needs to be removed. (b) The network connectivity for the outdoor deployment. (c)-(e) The states of nodes $u$ and $v$ (as labeled in (b)) which are connected and disconnected, respectively, from the source after the cut has occurred}
% \label{fig:outdoor_results}
% \end{figure}

In our deployment, the source strength was specified as $s=100$, the
iteration length was 5sec (this value could be reduced easily to as
small as~$200$ msec) and the cut detection threshold was $\epsilon=0.01$. Experimental results for two of the sensor nodes deployed are shown in Figure~\ref{fig:outdoor_results}. After about 30 iterations the states of all nodes converged. At iteration $k=83$ a cut is created by turning off motes inside the rectangle labeled ``Cut'' in Figure~\ref{fig:outdoor_results}(a). Figures \ref{fig:outdoor_results}(b) and \ref{fig:outdoor_results}(c) show the states for nodes $u$ and $v$, as depicted in Figure~\ref{fig:outdoor_results}(a), which were connected and disconnected, respectively, from the source node after the cut. The evolution of their states follows the aforementioned experimental scenario. Node $v$ declares itself cut from the source at $k=100$, since its state falls below the threshold $0.01$ at that time.

\subsection{Experimental Performance Evaluation in Mobile Network}\label{sec:exp-mobile}
For evaluating the performance of our separation detection system in
mobile networks we deployed 8 sensor nodes in an indoor environment,
with 4 of the nodes residing on Roomba robots and 4 on human
subjects. The scenario we emulated was that of a robotic-assisted
emergency response team. Figure~\ref{fig:mobile-exp-photo} shows part of the test
set-up with the mobile nodes.

The network topologies as well as the locations of the nodes at a few
time instants are shown in
Figure~\ref{fig:mobile_graph_experiment}. As we can see from the
figure, the topology of the network varied greatly over time due to
the mobility of the nodes.

Figure~\ref{fig:states-mobile-exp} shows the time-traces of the node
states during the experiment. The network is connected until time
$k=120$, and the states of all nodes converges to positive numbers;
see Figure~\ref{fig:states-mobile-exp}(a). This is consistent with the
prediction of theorem~\ref{thm:mobile-convergence}. At approximately
iteration $k=120$, four of the nodes (nodes 5 through 8), carried by
human subjects, are disconnected from the rest of the network, and in
particular, from the source node $1$. A sample network topology during
the time interval $k=120$ to $k=170$ is shown in
Figure~\ref{fig:mobile_graph_experiment}(b). As we can see
from~Figure~\ref{fig:states-mobile-exp}(b-c), the states of the
disconnected nodes $5$ through $8$ converge to zero. The nodes
$5,6,7,8$ detect that they are separated from the source, at times
$k=145,145,143,143$ respectively, when their states become lower than the
threshold $\epsilon  = 0.01$. At approximately iteration $k=170$, node $5$ joins back the sub-network formed by nodes $1$-$4$. As a result of node $5$ moving, node $7$ becomes a bridge between the two sub-networks. Hence, after iteration $k=170$, the states of nodes $6$, $7$ and $8$ become positive (hence, a fully connected network). However,
this re-connection is temporary, and nodes $6$ through $8$ again
become disconnected from the source after some time, which is seen in
their states. Another temporary connection occurs between the set of
nodes $6$-$8$ and the set of nodes $1$-$5$, during the time interval $k=180$
through $k=210$, followed by a separation. Finally, after iteration
$k=260$, the network becomes connected again, as shown
in~Figure~\ref{fig:mobile_graph_experiment}(d). As a result, the
states of all the nodes become positive after time $k=225$, and
they detect their re-connections to the source.

\section{Conclusions}\label{SECTION_Conclusions}
In this paper we introduced the Distributed Source
Separation Detection (DSSD) algorithm to detect network separation in
robotic and sensor networks. Simulations and hardware experiments demonstrated the
efficacy of the algorithm. DSSD requires communication only between neighbors,
which avoids the need for routing, making it particularly suitable for mobile
networks. The algorithm is distributed, doesn't require time
synchronization among nodes, and the computations involved are
simple. The DSSD algorithm is applicable to a heterogeneous network of
static as well as mobile nodes, with varying levels of resources,
precisely the kind envisioned for robotic and sensor networks.

%In spite of the simplicity of the idea behind the algorithm -- that of
%using a fictious potential to characterize network connectivity -- it convergence speed does not slow down with increasing network size, making it scalable to large networks.
%The idea behind the DSSD algoroithm -- that of
%using a fictious potential to characterize network connectivity --
%The algorithm may also be useful to be able to detect less
%severe changes to the connectivity structure that nonetheless affect the
%performance of a networked robotic system.

\bibliographystyle{model1a-num-names}
\bibliography{references}

\end{document}